\documentclass[10pt,a4paper]{article}
\usepackage[utf8]{inputenc}

\usepackage{amsmath}
\usepackage{amssymb}
\usepackage{amsthm}
\usepackage{amsfonts}
\usepackage{bm}
\usepackage{stmaryrd,mathtools}
\usepackage[langfont=typewriter,funcfont=slant]{complexity}
\usepackage[left=1in,bottom=1in,top=1in,right=1in]{geometry}
\usepackage[colorlinks=true]{hyperref}
\usepackage{framed}
\usepackage{algorithm}
\usepackage[noend]{algpseudocode}
\usepackage{soul}
\usepackage{thmtools,thm-restate}
\usepackage{csquotes}
\usepackage[nameinlink, noabbrev, capitalize]{cleveref}
\usepackage{booktabs}       
\usepackage{nicefrac}       
\usepackage{microtype}
\usepackage[dvipsnames, svgnames, x11names]{xcolor} 
\usepackage{enumerate,outlines}

\let\oldappendix\appendix
\renewcommand{\appendix}{%
  \oldappendix%
  \crefalias{section}{appendix}%
  \crefalias{subsection}{appendix}%
}

\hypersetup{
    colorlinks=true,
    linkcolor=Mulberry,       
    citecolor=DarkBlue,       
    urlcolor=DarkBlue         
}

\usepackage[disable]{todonotes} 
\presetkeys{todonotes}{inline}{}

\makeatletter
\if@todonotes@disabled

\else

\fi
\makeatother

\usepackage{commath} 

\usepackage{framed}

\theoremstyle{definition}

\theoremstyle{remark}
\newtheorem*{remark}{Remark}

\theoremstyle{plain}
\newtheorem{theorem}{Theorem}[section]

\theoremstyle{plain}

\theoremstyle{plain}

\theoremstyle{plain}
\newtheorem{lemma}[theorem]{Lemma}

\newclass{\NPH}{NPH}

\newcommand{\Secref}[1]{\hyperref[#1]{Section \ref*{#1}}}
\newcommand{\Appref}[1]{\hyperref[#1]{Appendix \ref*{#1}}}

\crefname{equation}{}{}
\crefrangeformat{equation}{(#3#1#4) --~(#5#2#6)}
\crefmultiformat{equation}{(#2#1#3)}{, (#2#1#3)}{, (#2#1#3)}{, (#2#1#3)}
\crefname{lemma}{Lemma}{Lemmas}
\crefname{section}{Section}{Sections}
\crefname{subsubsubsection}{Section}{Sections}
\crefname{remark}{Remark}{Remarks}
\crefname{figure}{Figure}{Figures}
\crefname{table}{Table}{Tables}
\Crefname{lemma}{Lemma}{Lemmas}
\crefname{theorem}{Theorem}{Theorems}
\Crefname{theorem}{Theorem}{Theorems}

\DeclareMathOperator*{\argmin}{\arg\!\min}

\newcommand{\Eb}{\mathbb{E}}
\newcommand{\Ib}{\mathbb{I}}
\newcommand{\Nb}{\mathbb{N}}
\newcommand{\Pb}{\mathbb{P}}

\newcommand{\Ac}{\mathcal{A}}
\newcommand{\Fc}{\mathcal{F}}

\newcommand{\Uc}{\mathcal{U}}
\newcommand{\Xc}{\mathcal{X}}

\newcommand{\Dsf}{\mathsf{D}}
\newcommand{\Ysf}{\mathsf{YES}}
\newcommand{\Nsf}{\mathsf{NO}}
\newcommand{\Tsf}{\mathsf{T}}
\newcommand{\OPT}{\mathsf{OPT}}
\newcommand{\RCN}{\mathsf{RCN}}
\newcommand{\MM}{\mathsf{MM}}

\newcommand{\SHT}{\mathsf{SHT}}
\newcommand{\MSHT}{\mathsf{MSHT}}
\newcommand{\Csf}{\mathsf{C}}
\newcommand{\Lsf}{\mathsf{L}}

\newcommand{\qbf}{\mathbf{q}}

\DeclareMathOperator{\Ber}{Ber}                     		
\DeclareMathOperator{\HG}{HG}                        		
\DeclareMathOperator{\Bin}{Bin}                        		
\DeclareMathOperator{\Pois}{Pois}                           

\DeclareMathOperator{\err}{err}
\DeclareMathOperator{\emperr}{\widehat{err}}

\DeclareMathOperator{\TV}{TV}
\DeclareMathOperator{\Var}{Var}
\DeclareMathOperator{\VC}{VCdim}

\newcommand{\iid}{\overset{iid}{\sim}}
\newcommand{\tm}[1]{^{(#1)}}
\DeclarePairedDelimiter{\llbr}{\llbracket}{\rrbracket}
\DeclarePairedDelimiter\ceil{\lceil}{\rceil}

\DeclareMathOperator{\ERM}{ERM}
\DeclareMathOperator{\Test}{Test}
\DeclareMathOperator{\RCNTest}{RCN-Test}
\DeclareMathOperator{\MMTest}{MM-Test}

\DeclareMathOperator{\Cond}{Cond}

\newif\ifcolt 
\coltfalse

\title{Is Multi-Distribution Learning as Easy as PAC Learning: \\ 
Sharp Rates with Bounded Label Noise}
\author{
    Rafael Hanashiro \\ MIT \\ \texttt{rafah@mit.edu}
    \and
    Abhishek Shetty \\ MIT \\ \texttt{shetty@mit.edu}
    \and
    Patrick Jaillet \\ MIT \\ \texttt{jaillet@mit.edu} 
}

\begin{document}
\maketitle

\begin{abstract}

Towards understanding the statistical complexity of learning from heterogeneous sources, we study the problem of multi-distribution learning. Given $k$ data sources, the goal is to output a classifier for each source by exploiting shared structure to reduce sample complexity. We focus on the bounded label noise setting to determine whether the fast $1/\epsilon$ rates achievable in single-task learning extend to this regime with minimal dependence on $k$.
Surprisingly, we show that this is not the case. We demonstrate that learning across $k$ distributions inherently incurs slow rates scaling with $k/\epsilon^2$, even under constant noise levels, unless each distribution is learned separately. A key technical contribution is a structured hypothesis-testing framework that captures the statistical cost of certifying near-optimality under bounded noise—a cost we show is unavoidable in the multi-distribution setting.

Finally, we prove that when competing with the stronger benchmark of each distribution's optimal Bayes error, the sample complexity incurs a \textit{multiplicative} penalty in $k$. This establishes a \textit{statistical} separation between random classification noise and Massart noise, highlighting a fundamental barrier unique to learning from multiple sources.




\end{abstract}
\section{Introduction}

\ifcolt\else Data heterogeneity inherent in learning from multiple sources is a fundamental challenge in modern machine learning. In federated learning, data resides across numerous devices or institutions~\cite{kairouz2021advances}; in algorithmic fairness, models must perform well across diverse demographic groups~\cite{Sagawa*2020-group-dro}; in domain adaptation, one must reason about distinct but related distributions~\cite{ben2010theory}; and in language modeling, the composition of pretraining data sources is critical. These applications have sparked significant interest in understanding how shared structure can be leveraged to minimize data requirements, with the ideal goal being that learning across heterogeneous sources incurs only a mild statistical overhead.

This motivates the study of \textit{Multi-Distribution Learning (MDL)}, where a learner must perform well across $k$ distributions simultaneously. Remarkably, in both the realizable and agnostic settings, MDL admits sample complexities of $\tilde\Theta\del{\frac{d+k}{\epsilon}}$ and $\tilde\Theta\del{\frac{d+k}{\epsilon^2}}$, respectively, for classes of VC dimension $d$~\cite{blum-collab-pac,nika-on-demand,zhang-optim-mdl,peng-mdl}. This reflects only an \textit{additive} overhead in $k$ relative to single-distribution learning, suggesting that in these classical settings, learning from heterogeneous sources is essentially ``as easy as PAC learning.''

However, the realizable and agnostic settings represent extremes of label quality. A more realistic scenario involves labels that are corrupted, but not entirely adversarially. This is captured by the \textit{bounded noise} model~\cite{Massart2006-rcn}, which assumes there exists a function $f^*\in\Fc$ such that\footnote{We will often write $P\del{\del{X,Y}\in A}$ as shorthand for $\Pb_{\del{X,Y}\sim P}\del{\del{X,Y}\in A}$.}
\begin{align*}
    P\del{f^*\del{X}\neq Y\vert X=x} \leq \eta < \frac{1}{2}
\end{align*}
for all $x\in\Xc$. Here, $f^*$ is the \textit{Bayes classifier} and attains the minimum possible error. In the standard single-distribution PAC setting~\cite{valiant-pac}, empirical risk minimization (ERM) achieves a sample complexity of $\tilde O\del{\frac{d}{\epsilon\del{1-2\eta}}}$. Notably, if we treat $\eta$ as constant, this recovers the fast $1/\epsilon$ realizable rate. Given the precedent that MDL incurs only a small overhead in $k$, one might hope that these fast rates persist under label noise as well.

In this work, we answer this question in the negative: we show that under bounded label noise, the cost of handling multiple distributions can be significantly higher than in the single-distribution setting.

 \fi

\paragraph{Notation.}
We define $\sbr{n} \coloneqq \cbr{1,2,\dots,n}$ and $\llbr{n} \coloneqq \cbr{0,1,\dots,n}$. We use $\lesssim$ and $\gtrsim$ to denote inequalities up to universal constants, and $\tilde O,\tilde \Omega, \tilde\Theta$ to hide logarithmic factors. For a distribution $P$ over $\Xc\times\cbr{0,1}$ and a function $f:\Xc\to\cbr{0,1}$, we define the classification error $\err\del{f;P} \coloneqq P\del{f\del{X}\neq Y}$.

\ifcolt\else \subsection{Problem Setup}
We assume sample access to $k$ distributions $P_1,\dots,P_k$ over $\Xc\times\cbr{0,1}$, and assume there exists a \textit{shared} Bayes classifier $f^*\in\Fc$ such that, for each $i\in\sbr{k}$ and $x\in\Xc$,
\begin{align*}
    P_i\del{f^*\del{X}\neq Y\middle\vert X=x} \leq \eta_i < \frac{1}{2}
\end{align*}
We assume that the noise upper bounds $\eta_1,\dots,\eta_k$ are known, but the optimal errors $\eta_i^* \coloneqq \err\del{f^*;P_i} \leq \eta_i$ are unknown. Note that setting all $\eta_i$ to $0$ recovers the standard realizable setting, whereas removing the existence assumption on $f^*$ yields the agnostic regime.

Our goal is to adaptively sample from the $P_i$ and output decisions $\hat f_1,\dots,\hat f_k:\Xc\to\cbr{0,1}$ that compete with prescribed benchmarks $\OPT_1,\dots,\OPT_k$, in the sense that
\begin{align}
    \Pb\del{ \max_{i\in\sbr{k}}\cbr{ \err\del{\hat f_i;P_i} - \OPT_i } \leq \epsilon } \geq 1-\delta \label{eq:gen-MDL-obj}
\end{align}
This is known as the \textit{personalized} setting, as opposed to \textit{centralized}, since we are allowed to output a different decision per distribution. We focus on three separate regimes.

\paragraph{RCN.}
Recall that under RCN, the pointwise error is constant, $P_i\del{f^*\del{X}\neq Y\middle\vert X=x} = \eta_i$, implying the optimal error is exactly $\eta_i$. Motivated by this, we introduce the benchmark:
\begin{align}
    &\OPT_i = \eta_i \tag{MDL-RCN}\label{eq:MDL-RCN}
\end{align}
Here, the learner aims to compete with the known upper bounds $\eta_i$, even though the true optimal errors $\eta_i^*$ may be strictly smaller. We retain distinct indices $\eta_i$ for completeness, but note that their heterogeneity does not fundamentally alter the hardness of the problem; one may effectively treat them as equal.

\paragraph{Minimax.}
The standard benchmark in MDL is the minimax risk $\eta^* \coloneqq \max_{i\in\sbr{k}}\eta_i^*$, corresponding to the best worst-case error a single hypothesis can achieve across all distributions. We adopt this as our second benchmark:
\begin{align}
    &\OPT_i = \eta^* \tag{MDL-MM}\label{eq:MDL-MM}
\end{align}

\paragraph{Massart.}
Finally, we consider the most fine-grained objective, where the learner must compete with the true optimal error of each distribution individually:
\begin{align}
    &\OPT_i = \eta_i^* \tag{MDL-Mass}\label{eq:MDL-Mass}
\end{align}
Unlike the minimax setting, this requires the learner to adapt to the specific noise level of every distribution. Evidently, this variant is at least as hard as the other two, as the learner targets the strictest possible benchmark for each instance.

For any regime, we say that an algorithm $\Ac$ has sample complexity $T_\Ac:\del{0,1}^{2k+2}\to\Nb$ if, given any instance $\del{\Fc, P_{1:k}, \eta^*_{1:k}, \eta_{1:k},\epsilon,\delta}$, it satisfies~\eqref{eq:gen-MDL-obj} using at most $T_\Ac\del{\eta_{1:k}^*, \eta_{1:k}, \epsilon, \delta}$ samples in total across all distributions. We omit parameters from the input when there is no dependence on them. For convenience, we additionally define $\eta\coloneqq\max_{i\in\sbr{k}}\eta_i$. \fi

\ifcolt\else \subsection{Overview and Contributions}
We provide a near-complete characterization of the sample complexity for MDL under bounded label noise. Along the way, we introduce a structured hypothesis testing problem that may be of independent interest.

\paragraph{Upper Bounds (\Cref{sec:ub}).}
We develop a meta-algorithm (\Cref{alg:noisy-pers}) for the first two variants, following the iterative approach of~\cite{blum-collab-pac}. In each round, the algorithm learns at least half of the active distributions, subsequently identifying them via a testing subroutine. Under~\eqref{eq:MDL-RCN}, the target $\eta_i$ are known, allowing for direct certification of optimality. This yields an upper bound
\begin{align*}
    \tilde O\del{ \frac{d}{\epsilon\del{1-2\eta}}+\sum_{i=1}^k \frac{\epsilon+\eta_i}{\epsilon^2} } \tag*{(\Cref{thm:MDL-RCN-ub})}
\end{align*}
The minimax variant is more subtle, as the benchmark $\eta^*$ is unknown. To address this, we devise a procedure that estimates $\eta^*$ in increments and incurs only a logarithmic overhead. This results in an~\eqref{eq:MDL-MM} upper bound
\begin{align*}
    \tilde O\del{ \frac{d}{\epsilon\del{1-2\eta}}+\frac{k\del{\epsilon+\eta^*}}{\epsilon^2} } \tag*{(\Cref{thm:MDL-MM-ub})}
\end{align*}
In both settings, the first term reflects the statistical cost of learning, while the second captures the testing complexity. Since the latter scales quadratically with $\epsilon$, it can be preferable to learn each distribution independently, which would require $\tilde O\del{\sum_{i=1}^k\frac{d}{\epsilon\del{1-2\eta_i}}}$ samples. \Cref{alg:noisy-pers} explicitly compares these regimes (via an initial check) and switches to per-distribution learning whenever it is more sample-efficient. Note that in the realizable case ($\eta_i=0$ for all $i$), we recover the rate $\tilde O\del{\frac{d+k}{\epsilon}}$.

The testing component arises naturally from the learning objective, as the ability to learn an optimal classifier intuitively implies the ability to verify its performance. However, this certification incurs a statistical cost of $\tilde O\del{\frac{\epsilon+\eta}{\epsilon^2}}$ to reliably distinguish between noise rates $\eta$ and $\eta+\epsilon$. Establishing the necessity of this cost constitutes the more challenging direction. Our lower bounds confirm that it is not an artifact of analysis, but a fundamental barrier in learning under bounded noise.

\paragraph{Structured Hypothesis Testing (\Cref{sec:SHT}).}
As discussed, the ability to test whether a classifier is optimal is intrinsic to the learning problem. To formalize this task, we introduce \textit{Structured Hypothesis Testing}~\eqref{eq:SHT}: given a fixed function, decide whether it is $\epsilon$-optimal under RCN with known noise $\eta=1/4$.

This problem admits two natural strategies: (i) directly thresholding the empirical errors, or (ii) first learning $f^*$ to compare against the candidate. The former succeeds with $\tilde O\del{1/\epsilon^2}$ samples, while the latter requires $\tilde O\del{d/\epsilon}$. We demonstrate that this trade-off is fundamental by establishing a matching lower bound, yielding the optimal sample complexity:
\begin{align*}
    \tilde\Theta\del{\min\cbr{\frac{1}{\epsilon^2}, \frac{d}{\epsilon}}} \tag*{(\Cref{thm:SHT-ub,thm:SHT-lb})}
\end{align*}
We then extend this paradigm to \textit{Multi-Distribution Structured Hypothesis Testing}~\eqref{eq:MSHT}, where the goal is to solve $k$ simultaneous instances of~\eqref{eq:SHT} across distributions $P_1,\dots,P_k$ that share a Bayes classifier with fixed noise rate $\eta=1/4$. We prove that the naive strategy of testing each distribution separately is, in fact, optimal.

To establish the lower bound, we reduce~\eqref{eq:SHT} to~\eqref{eq:MSHT} by embedding a hard~\eqref{eq:SHT} instance into a specific $P_i$, while setting the remaining distributions to the null hypothesis. Since the index $i$ is unknown, the~\eqref{eq:MSHT} algorithm is forced to sample sufficiently from all distributions to locate and solve the embedded instance. Crucially, we construct the~\eqref{eq:SHT} hard instance such that extending it to multiple distributions preserves the VC dimension of the underlying hypothesis class. This yields the final sample complexity:
\begin{align*}
    \tilde\Theta\del{k\cdot\min\cbr{\frac{1}{\epsilon^2}, \frac{d}{\epsilon}}} \tag*{(\Cref{thm:MSHT-ub,thm:SHT-lb})}
\end{align*}

\paragraph{MDL Lower Bounds (\Cref{sec:mdl-lb}).}
As discussed, testing optimality is a necessary condition for learning. In \Cref{lem:MDL-MSHT-ub}, we formalize this by showing that any MDL algorithm can be used to solve~\eqref{eq:MSHT} with an additional cost of only $\tilde O\del{k/\epsilon}$ samples. Consequently, in the fixed-noise regime where $\eta^*_i=\eta_i=1/4$ for all $i\in\sbr{k}$, every MDL variant requires
\begin{align*}
    \Omega\del{\frac{d}{\epsilon} + k\cdot\min\cbr{\frac{1}{\epsilon^2},\frac{d}{\epsilon}}} \tag*{(\Cref{thm:MDL-lb})}
\end{align*}
samples. When $d\lesssim 1/\epsilon$, this becomes $\Omega\del{dk/\epsilon}$---no better than learning each distribution separately. This establishes optimality of our algorithms for~\eqref{eq:MDL-RCN} and~\eqref{eq:MDL-MM}, and answers our motivating question in the negative.

\paragraph{Separation Between RCN and Massart (\Cref{sec:mdl-mass}).}
For the benchmark in~\eqref{eq:MDL-Mass}, where each output $\hat f_i$ must compete with the \textit{unknown} Bayes error $\eta_i^*$, we prove the stronger lower bound
\begin{align*}
    \Omega\del{\frac{d}{\epsilon} + k\cdot \min\cbr{\frac{1}{\epsilon^2},\frac{d}{\epsilon}} + \frac{k\sqrt{d}}{\epsilon}} \tag*{(\Cref{thm:MDL-Mass-lb})}
\end{align*}
under the assumption that $\eta_i\leq0.49$ for all $i\in\sbr{k}$. Beyond the learning term $d/\epsilon$ and the intrinsic testing cost $k\cdot \min\cbr{1/\epsilon^2,\, d/\epsilon}$ already necessary in the fixed-noise regime, we show an additional penalty of $k\sqrt d/\epsilon$ that is unique to this Massart objective. Informally, when the target $\eta_i^*$ is not given, an MDL learner must certify near-optimality relative to an \textit{unknown} baseline on each distribution, and this extra uncertainty forces a $\sqrt d$-scaled sample overhead for each of the $k$ tasks. Consequently, while RCN and Massart noise are statistically comparable in the single-distribution setting, they become separable in MDL: competing with $\eta_i^*$ is strictly harder and necessarily amplifies the dependence on $k$. This separation has long been known computationally and is closely tied to distribution shift~\cite{sitan-mass}. We interpret the present statistical separation in MDL as shedding new light on this phenomenon. \fi
\section{Related Work}

\paragraph{Multi-Distribution Learning.}
Multi-distribution learning was introduced by \cite{blum-collab-pac} under the name \textit{collaborative PAC learning}. In the realizable case, they proved a tight sample complexity of $\tilde \Theta\del{\frac{d+k}{\epsilon}}$ under both personalized and centralized settings.

In the centralized realizable setting, \cite{Nguyen_Zakynthinou_2018-improved-pac,Chen_Zhang_Zhou_2018-tight-pac} improved the logarithmic dependence, and \cite{Nguyen_Zakynthinou_2018-improved-pac} also gave strategies that compete with a constant multiple of the minimax error.

For centralized agnostic MDL, \cite{nika-on-demand} proposed a general algorithmic framework based on playing online learning strategies against each other to reach an approximate Nash equilibrium. They obtained a tight rate of $\tilde \Theta\del{\frac{\log\abs{\Fc}+k}{\epsilon^2}}$ for finite classes and, for hypothesis classes of VC dimension $d$, an upper bound of $\tilde O\del{\frac{dk}{\epsilon} + \frac{d+k}{\epsilon^2}}$; for the VC case they also proved a lower bound of $\tilde \Omega\del{\frac{d+k}{\epsilon^2}}$.

Subsequently, \cite{mdl-open-prob} asked whether this VC lower bound is tight for centralized agnostic MDL. They also provided an algorithm for personalized agnostic MDL with complexity $\tilde O\del{\frac{d+k}{\epsilon^2}}$ and an algorithm for centralized agnostic MDL with complexity $\tilde O\del{\frac{d}{\epsilon^4} + \frac{k}{\epsilon^2}}$. This question was later answered affirmatively by \cite{zhang-optim-mdl,peng-mdl}, who gave algorithms for centralized agnostic MDL with sample complexity $\tilde O\del{\frac{d+k}{\epsilon^2}}$, matching the lower bound up to logarithmic factors.

Closest to our setting, \cite{deng-diff-lab} studied a weaker form of realizability where a small subset of classifiers achieves small error: they assume there exist $f_1^*,\dots,f_m^*\in\Fc$ such that $\max_{i\in\sbr{k}} \min_{j\in\sbr{m}} \err\del{f_j^*;P_i} \leq \epsilon$. They gave an algorithm with sample complexity $\tilde O\del{\frac{md}{\epsilon} + \frac{k}{\epsilon}}$ and showed it is tight up to log factors. In particular, when different distributions are realized by different classifiers, a learner must in the worst case learn each distribution separately; in contrast, we assume a \textit{shared} Bayes classifier, a structural distinction that enables information sharing not possible in the setting of \cite{deng-diff-lab}.

\paragraph{Related Learning Settings.}
The MDL framework has also been studied under a variety of structural assumptions---including convexity~\cite{Sagawa*2020-group-dro,abernethy22a-minmax-fairness,nika-on-demand,soma-dro,zhang-stoch-approx-gdro,yu-eff-emp-gdro,carmon-ball-oracle,zang-adapt-data-coll}, sparsity~\cite{nguyen-sparse-mdl}, and bounded suboptimality~\cite{hanashiro2025distributiondependent}---as well as through alternative lenses such as label complexity~\cite{zhang-active-mdl,rittler-agn-mg-al} and the role of adaptivity in sample complexity~\cite{haghtalab-sample-adapt}. Beyond these theoretical directions, distributionally robust objectives have also been adopted in language models to improve performance across diverse data sources~\cite{oren-etal-2019-dr-lm,xie2023doremi}.

On the computational side, a complementary line of work studies the cost of making statistically optimal procedures efficient: since the optimal strategies are inherently randomized, \cite{larsen-derand-mdl} showed that derandomizing them is computationally hard. Related notions of learning across heterogeneous groups have also been investigated in other learning frameworks~\cite{mohri19a-agnostic-fl,dwork2025how,rothblum-multi-group-agn-pac}.

\paragraph{PAC Learning.}
Statistical learning theory has long been developed through the lens of PAC learning~\cite{valiant-pac,haussler-pac}. Classical formulations focus either on the realizable case (zero noise) or on the fully agnostic case (no assumptions on the data-generating process). To model intermediate regimes, subsequent works~\cite{Angluin1988-oa,kearns-rcn,Mammen1999-uz} introduced explicit assumptions on the label noise, which can often lead to faster learning rates; see~\cite{Boucheron2005-il} for background.

In this work, we focus on the RCN and Massart noise models. From a statistical perspective, the distinction between the two is minor; both admit fast rates. Computationally, however, the Massart generalization proves substantially more challenging. A polynomial-time algorithm for RCN was given by~\cite{bylander-rcn}, whereas a distribution-independent equivalent for Massart noise remained elusive until much later~\cite{diak-mass}. Furthermore, current polynomial-time approaches generally compete only with the noise upper bound, and~\cite{sitan-mass} established a super-polynomial statistical query SQ lower bound for competing with the Bayes error.
\section{MDL Upper Bounds}
\label{sec:ub}

\ifcolt\else In this section, we develop strategies for~\eqref{eq:MDL-RCN} and~\eqref{eq:MDL-MM}. Let $d$ denote the VC dimension of the hypothesis class $\Fc$. \cref{alg:noisy-pers} outlines the general template for our approach. At a high level, the learner proceeds in round $t$ as follows:

\begin{itemize}
\item \textbf{Active Set Maintenance:} The learner maintains a set of distributions $\Uc\tm{t-1}$ that still need to be learned, initialized with the full set $\Uc\tm{0} \coloneqq \cbr{P_1,\dots,P_k}$.

\item \textbf{Statistical Learning:} It learns a hypothesis $f\tm{t}$ with respect to the uniform mixture over the current active set, denoted by $\bar P_{\Uc\tm{t-1}}$, using an ERM oracle (Lines~\ref{code:MDL-SL-start}--\ref{code:MDL-SL-end}):
\begin{align*}
    \ERM_\Fc\del{S} \in \argmin_{f\in\Fc} \cbr{ \emperr\del{f;S} \coloneqq \frac{1}{\abs{S}} \sum_{\del{x,y}\in S} \Ib\cbr{f\del{x}\neq y} }
\end{align*}

\item \textbf{Testing:} It invokes a $\Test$ subroutine (specific to each task) to identify the distributions on which $f\tm{t}$ performs well. These distributions are removed from $\Uc\tm{t-1}$ to form the next active set $\Uc\tm{t}$ (Lines~\ref{code:MDL-test-start}--\ref{code:MDL-test-end}).
\end{itemize}
We will show that the testing component can require $\Omega\del{k/\epsilon^2}$ samples. Due to this high testing overhead, it is more efficient to learn each distribution separately whenever $d \ll 1/\epsilon$. To address this, \Cref{alg:noisy-pers} incorporates a preliminary check to determine if separate learning yields better complexity. \fi

\ifcolt
\else

\begin{algorithm}
\caption{MDL Meta-Algorithm}
\label{alg:noisy-pers}
\begin{algorithmic}[1]
    \Require Function class $\Fc$, distributions $P_1,\dots,P_k$, total iteration count $T$, parameters $\del{\epsilon, \epsilon', \delta, \delta'}\in\del{0,1}^4$, condition $\Cond$, tester function $\Test$.
    \If{$\Cond$} \Comment{Learn each distribution separately.}
        \For{$i=1,\dots,k$}
            \State Sample $S_i\iid P_i$ of size $\abs{S_i} = T_\SL\del{\eta_i,\epsilon,\delta/k}$.
            \State $\hat f_i \gets \ERM_\Fc\del{S_i}$.
        \EndFor
    \Else
        \State $\Uc\tm{0} \gets \cbr{P_1,\dots,P_k}$.
        \For{$t=1,\dots,T$}
            \State\label{code:MDL-SL-start} Sample $S\tm{t}\iid \bar P_{\Uc\tm{t-1}}$ of size $\abs{S\tm{t}} = T_\SL\del{\eta,\epsilon',\delta'}$.
            \State\label{code:MDL-SL-end} $f\tm{t} \gets \ERM_\Fc\del{S\tm{t}}$. 
            \State\label{code:MDL-test-start} $\Uc\tm{t} \gets \Test\del{f\tm{t},\Uc\tm{t-1}}$.
            \For{$i\in\Uc\tm{t-1}\backslash\Uc\tm{t}$}
                \State\label{code:MDL-test-end} $\hat f_i \gets f\tm{t}$.
            \EndFor
        \EndFor
    \EndIf
    \Ensure $\hat f_1,\dots,\hat f_k$.
\end{algorithmic}
\end{algorithm}

\fi


\ifcolt\else \subsection{Statistical Learning}
Under label noise bounded by $\eta$, ERM achieves $(\epsilon,\delta)$-learning with sample complexity
\begin{align*}
    T_\SL\del{\eta,\epsilon,\delta} \coloneqq C_\SL \cdot \frac{d\log\del{1/\epsilon} + \log\del{1/\delta}}{\epsilon\del{1-2\eta}}
\end{align*}
where $C_\SL$ is a universal constant. We provide further details in \Cref{app:noisy-class} and refer the reader to~\cite{Boucheron2005-il} for a comprehensive treatment. We leverage this guarantee to determine the sample size for the statistical learning component of \Cref{alg:noisy-pers}. Specifically, to ensure that the hypothesis $f\tm{t}$ is $\epsilon$-optimal with respect to the mixture distribution $\bar P_{\Uc\tm{t-1}}$, we rely on the following result.

\begin{lemma}
\label{lem:sl-opt}
Let $\Uc\subset\cbr{P_1,\dots,P_k}$ and let $\bar P_\Uc = \frac{1}{\abs{\Uc}} \sum_{i\in\Uc} P_i$ be their uniform mixture. Then, ERM $\hat f = \ERM_\Fc\del{S}$ on a sample $S\iid \bar P_\Uc$ of size $\abs{S} = T_\SL\del{\eta,\epsilon,\delta}$ satisfies
\begin{align*}
    \Pb\del{ \err\del{\hat f; \bar P_\Uc} \leq \epsilon + \frac{1}{\abs{\Uc}}\sum_{i\in\Uc} \eta_i^* } \geq 1-\delta
\end{align*}
\end{lemma}

Applying this bound to \Cref{alg:noisy-pers}, we conclude that for any fixed iteration $t\in\sbr{T}$,
\begin{align*}
    \Pb\del{ \err\del{f\tm{t};\bar P_{\Uc\tm{t-1}}} \leq \epsilon' + \frac{1}{\abs{\Uc\tm{t-1}}}\sum_{i\in\Uc\tm{t-1}} \eta_i^* } \geq 1-\delta'
\end{align*}
Conditioned on this high-probability event, we can rewrite the inequality in terms of the average excess error:
\begin{align*}
    \frac{1}{\abs{\Uc\tm{t-1}}}\sum_{i\in\Uc\tm{t-1}} \del{ \err\del{f\tm{t};P_i} - \eta_i^* } \leq \epsilon'
\end{align*}
Now, let $k\tm{t} \coloneqq \abs{\cbr{i\in\Uc\tm{t-1}: \err\del{f\tm{t};P_i} - \eta_i^* \geq2\epsilon'}}$ denote the number of distributions in $\Uc\tm{t-1}$ on which $f\tm{t}$ performs poorly. Bounding the average excess error from below by the contribution of these $k\tm{t}$ distributions yields
\begin{align*}
    2\epsilon'\cdot\frac{k\tm{t}}{\abs{\Uc\tm{t-1}}} \leq \frac{1}{\abs{\Uc\tm{t-1}}} \sum_{i\in\Uc\tm{t-1}} \del{\err\del{f\tm{t};P_i}-\eta_i^*} \leq \epsilon' \Longrightarrow k\tm{t} \leq \frac{\abs{\Uc\tm{t-1}}}{2}
\end{align*}
Consequently, $f\tm{t}$ is $2\epsilon'$-optimal for at least half of the distributions in $\Uc\tm{t-1}$. \fi
\ifcolt\else \subsection{Testing}
While the statistical learning step guarantees that $f\tm{t}$ has low error on at least half of the distributions in the active set $\Uc\tm{t-1}$, the identity of these distributions remains unknown. The following result demonstrates that we can identify them by estimating the empirical error on samples drawn from each distribution.

\begin{lemma}
\label{lem:test-comp}
Let $P$ be a distribution over $\Xc\times\cbr{0,1}$, $f:\Xc\to\cbr{0,1}$ be some function, and $\nu\geq 0$ be a constant. Let $S\iid P$ be a sample of size
\begin{align*}
    \abs{S}\geq T_\Tsf\del{\nu,\epsilon,\delta} \coloneqq 16\log\del{\frac{2}{\delta}} \frac{\epsilon+8\nu}{\epsilon^2}
\end{align*}
Then, with probability at least $1-\delta$, the following holds:
\begin{align*}
    \err\del{f;P}\leq\frac{\epsilon}{8}+\nu \Longrightarrow \emperr\del{f;S}\leq\frac{\epsilon}{2}+\nu \Longrightarrow \err\del{f;P}\leq\epsilon+\nu
\end{align*}
\end{lemma}

In essence, this implies that by sampling $\tilde O\del{\frac{\epsilon+\nu}{\epsilon^2}}$ examples from a distribution $P$, we can reliably distinguish whether the error $\err\del{f;P}$ is $\epsilon$-close to a target level $\nu$. \fi
\ifcolt\else \subsection{\texorpdfstring{\eqref{eq:MDL-RCN}}{(MDL-RCN)} Upper Bound}
When our goal is to compete with the noise upper bounds $\eta_i$, the testing problem is simple because the target thresholds are known. Fix a round $t\in\sbr{T}$. Recall from the statistical learning step that with probability at least $1-\delta'$, for at least half of the distributions in $\Uc\tm{t-1}$,
\begin{align*}
    \err\del{f\tm{t};P_i} \leq 2\epsilon' + \eta_i^* \leq 2\epsilon' + \eta_i
\end{align*}
Since each $\eta_i$ is known, we can verify this condition directly. For each distribution $i\in\Uc\tm{t-1}$, we draw a sample $S\tm{t}_i\iid P_i$ of size $\abs{S_i\tm{t}} = T_\Tsf\del{\eta_i,\epsilon,\delta''}$. Applying \Cref{lem:test-comp} together with a union bound, we guarantee that with probability at least $1-k\delta''$, the following implications hold for all $i$:
\begin{align*}
    \err\del{f\tm{t};P_i}\leq\frac{\epsilon}{8}+\eta_i \Longrightarrow \emperr\del{f\tm{t};S_i\tm{t}} \leq \frac{\epsilon}{2}+\eta_i \Longrightarrow \err\del{f\tm{t};P_i}\leq\epsilon+\eta_i
\end{align*}
We set $\epsilon' = \epsilon/16$ and define the update rule as follows: eliminate distribution $i$ if and only if its empirical error satisfies $\emperr\del{f\tm{t};S_i\tm{t}} \leq \epsilon/2+\eta_i$. This strategy guarantees two key properties:
\begin{enumerate}
    \item \textbf{(Progress)} At least half of the distributions in $\Uc\tm{t-1}$ (specifically those with $\err\del{f\tm{t};P_i} \leq \epsilon/8+\eta_i$) will satisfy the empirical condition and be eliminated.

    \item \textbf{(Correctness)} Any eliminated distribution is guaranteed to satisfy $\err\del{f\tm{t};P_i}\leq\epsilon+\eta_i$.
\end{enumerate}
Thus, in each round $t$, we successfully learn and remove at least half of the active distributions with probability $1-\delta'-k\delta''$. Appropriately tuning the parameters yields the procedure described in \Cref{alg:rcn-test}.

\ifcolt

\begin{algorithm}
\DontPrintSemicolon
\setcounter{AlgoLine}{0}
\caption{$\RCNTest$}
\label{alg:rcn-test}
\Input{Decision $f\tm{t}$, distribution set $\Uc\tm{t-1}$.}
\Output{$\Uc\tm{t}$.}
\lFor{$i\in\Uc\tm{t-1}$}{
    Sample $S_i\tm{t}\iid P_i$ of size $\abs{S_i\tm{t}} = T_\Tsf\del{\eta_i,\epsilon,\frac{\delta}{2kT}}$
}
$\Uc\tm{t} \gets \cbr{ i\in\Uc\tm{t-1}: \emperr\del{f\tm{t};S_i\tm{t}}>\epsilon/2+\eta_i }$
\end{algorithm}

\else

\begin{algorithm}
\caption{$\RCNTest$}
\label{alg:rcn-test}
\begin{algorithmic}[1]
    \Require Decision $f\tm{t}$, distribution set $\Uc\tm{t-1}$.
    \For{$i\in\Uc\tm{t-1}$}
        \State Sample $S_i\tm{t}\iid P_i$ of size $\abs{S_i\tm{t}} = T_\Tsf\del{\eta_i,\epsilon,\frac{\delta}{2kT}}$.
    \EndFor
    \State $\Uc\tm{t} \gets \cbr{ i\in\Uc\tm{t-1}: \emperr\del{f\tm{t};S_i\tm{t}}>\epsilon/2+\eta_i }$.
    \Ensure $\Uc\tm{t}$.
\end{algorithmic}
\end{algorithm}

\fi

\begin{theorem}[\eqref{eq:MDL-RCN} upper bound]
\label{thm:MDL-RCN-ub}
Let $T=\ceil*{\log_2 k}$, $\epsilon'=\frac{\epsilon}{16}$, $\delta'=\frac{\delta}{2T}$, and $\Test = \RCNTest$. Additionally, define the condition $\Cond = \cbr{T_\text{joint} > T_\text{sep}}$, where
\begin{align*}
    T_\mathsf{joint} \coloneqq \log^2\del{\frac{k}{\delta}}\del{\frac{d\log\del{1/\epsilon}}{\epsilon\del{1-2\eta}} + \sum_{i=1}^k \frac{\epsilon+\eta_i}{\epsilon^2}} \quad\text{and}\quad T_\mathsf{sep} \coloneqq \sum_{i=1}^k \frac{d\log\del{1/\epsilon} + \log\del{k/\delta}}{\epsilon\del{1-2\eta_i}}
\end{align*}
Then, \Cref{alg:noisy-pers} solves \eqref{eq:MDL-RCN} with sample complexity
\begin{align*}
    T_\RCN\del{\eta_{1:k}, \epsilon, \delta} = O\del{ \min\cbr{ T_\mathsf{joint}, T_\mathsf{sep} } }
\end{align*}
\end{theorem}

In the special case where the noise upper bound is uniform across distributions---that is, $\eta_i=\eta$ for all $i\in\sbr{k}$---the sample complexity bound from \Cref{thm:MDL-RCN-ub} simplifies (ignoring logarithmic factors) to
\begin{align*}
    \tilde O\del{ \min\cbr{ \frac{d}{\epsilon\del{1-2\eta}} + \frac{k\del{\epsilon+\eta}}{\epsilon^2}, \frac{dk}{\epsilon\del{1-2\eta}} } }
\end{align*} \fi
\ifcolt\else \subsection{\texorpdfstring{\eqref{eq:MDL-MM}}{(MDL-MM)} Upper Bound}
When our goal is to compete with the \textit{unknown} maximum noise rate $\eta^* = \max_{i\in\sbr{k}} \eta_i^*$, the previous testing strategy is inapplicable because we lack the explicit thresholds required for $\RCNTest$.

Fix a round $t\in[T]$. The statistical learning component guarantees that with probability at least $1-\delta'$, at least half of the distributions $i\in\Uc\tm{t-1}$ satisfy
\begin{align*}
    \err\del{f\tm{t};P_i} \leq 2\epsilon' + \eta_i^* \leq 2\epsilon' + \eta^*
\end{align*}
If $\eta^*$ was known, we could simply set $\epsilon'=\epsilon/16$ and invoke $\RCNTest$ using the uniform threshold $\epsilon/2+\eta^*$ for all $i\in\sbr{k}$. Since $\eta^*$ is unknown, we proceed by guessing its value in small increments.

Let $\nu$ denote a candidate guess for $\eta^*$. We set $\epsilon'=\epsilon/32$, ensuring that with probability at least $1-\delta'$, half of the active distributions satisfy $\err\del{f\tm{t};P_i} \leq \epsilon/16 + \eta^*$. \Cref{lem:test-comp} guarantees that by drawing a sample $S_i\tm{t}\iid P_i$ of size $\abs{S_i\tm{t}} = T_\Tsf\del{\nu,\epsilon/2,\delta''}$, the following holds with probability $1-k\delta''$: for every $i\in\sbr{k}$,
\begin{align*}
    \err\del{f\tm{t};P_i} \leq \frac{\epsilon}{16}+\nu \Longrightarrow \emperr\del{f\tm{t};S_i\tm{t}} \leq \frac{\epsilon}{4}+\nu \Longrightarrow \err\del{f\tm{t};P_i} \leq \frac{\epsilon}{2}+\nu
\end{align*}
Based on this, if we remove all $i\in\Uc\tm{t-1}$ (i.e., output $f\tm{t}$ for $P_i$) such that $\emperr\del{f\tm{t};S_i\tm{t}} \leq \epsilon/4+\nu$, we ensure the following:
\begin{enumerate}
    \item \textbf{(Progress)} Any distribution with true error $\err\del{f\tm{t};P_i} \leq \epsilon/16 + \nu$ will be removed.

    \item \textbf{(Correctness)} Any removed distribution satisfies $\err\del{f\tm{t};P_i} \leq \epsilon/2 + \nu$.
\end{enumerate}
We begin with the candidate guess $\nu=0$. Following the idea above, we draw samples $S_i\tm{t,1}\iid P_i$ of size $\abs{S_i\tm{t,1}} = T_\Tsf\del{0,\epsilon/2,\delta''}$ for each active distribution and eliminate those satisfying $\emperr\del{f\tm{t};S_i\tm{t,1}} \leq \epsilon/4$. If this step results in the removal of at least half of $\Uc\tm{t-1}$, we terminate the current round $t$. Crucially, the Correctness property ensures that any removed distribution satisfies $\err\del{f\tm{t};P_i}\leq\epsilon/2\leq\epsilon+\eta^*$, guaranteeing that we only output good hypotheses. Furthermore, if the true noise rate is indeed $\eta^*=0$, the Progress property guarantees that this round terminates. By union bounding with the ERM guarantee, all of the above holds with probability $1-\delta'-k\delta''$.

If the first check fails to eliminate half of the distributions, we proceed with the updated guess $\nu=\epsilon/2$. Rather than discarding the previous data, we augment the existing samples: for each remaining distribution $i$, we draw a fresh sample $\tilde S_i\tm{t,2}\iid P_i$ of size $\abs{\tilde S_i\tm{t,2}} = T_\Tsf\del{\epsilon/2,\epsilon/2,\delta''} - \abs{S_i\tm{t,1}}$. We then define $S_i\tm{t,2} = S_i\tm{t,1}\cup\tilde S_i\tm{t,2}$ as the aggregated sample set, which now satisfies the required size $\abs{S_i\tm{t,2}} = T_\Tsf\del{\epsilon/2,\epsilon/2,\delta''}$.

We eliminate any distribution satisfying $\emperr\del{f\tm{t};S_i\tm{t,2}} \leq \epsilon/4+\epsilon/2$, terminating the round if at least half of $\Uc\tm{t-1}$ has been removed. Observe that this threshold guarantees the removal of any $i$ where $\err\del{f\tm{t};P_i} \leq \epsilon/16+\epsilon/2$, due to Progress. Consequently, if the true noise rate lies in the range $\eta^* \in (0,\epsilon/2]$, the round terminates. Furthermore, Correctness ensures that any removed distribution satisfies $\err\del{f\tm{t};P_i} \leq \epsilon/2+\epsilon/2 \leq \epsilon + \eta^*$. Taking a union bound over both testing sub-rounds, all of the above holds with probability at least $1-\delta'-2k\delta''$.

We continue iteratively, testing the hypotheses that $\eta^*$ lies within consecutive intervals $\cbr{0}, (0,\epsilon/2], (\epsilon/2,\epsilon]$, and so on. Suppose that the true noise rate falls in the interval $\eta^*\in(\del{m-1}\epsilon/2,m\epsilon/2]$ and the procedure has not yet terminated in the first $m$ iterations. Then in the $\del{m+1}$th iteration, we set the guess $\nu=m\epsilon/2$. By assumption, we know that $m\leq2\eta^*/\epsilon+1$ and $\eta^* \leq m\epsilon/2 \leq \eta^* + \epsilon/2$. For each remaining distribution $i$, we sample $\tilde S_i\tm{t,m+1}\iid P_i$ such that the aggregated sample $S_i\tm{t,m+1} = S_i\tm{t,m}\cup\tilde S_i\tm{t,m+1}$ has size
\begin{align*}
    \abs{S_i\tm{t,m+1}} = T_\Tsf\del{m\epsilon/2,\epsilon/2,\delta''} = 32\log\del{\frac{2}{\delta''}} \frac{\epsilon+16m\epsilon/2}{\epsilon^2} \leq 32\log\del{\frac{2}{\delta''}} \frac{9\epsilon+16\eta^*}{\epsilon^2}
\end{align*}
We then remove any distribution satisfying $\emperr\del{f\tm{t};S_i\tm{t,m+1}} \leq \epsilon/4+m\epsilon/2$. Once again, the known properties yield the following:
\begin{itemize}
    \item \textbf{Progress:} We are guaranteed to remove any $i$ where $\err\del{f\tm{t};P_i} \leq \epsilon/16+m\epsilon/2$. Since we know that $\err\del{f\tm{t};P_i} \leq \epsilon/16+\eta^* \leq \epsilon/16+m\epsilon/2$ for at least half of $\Uc\tm{t-1}$, the procedure is guaranteed to terminate at this step.

    \item \textbf{Correctness:} Any distribution $i$ removed in this iteration satisfies
    \begin{align*}
        \err\del{f\tm{t};P_i} \leq \frac{\epsilon}{2}+\frac{m\epsilon}{2} = \epsilon + \frac{\del{m-1}\epsilon}{2} \leq \epsilon+\eta^*
    \end{align*}
\end{itemize}
Consequently, the testing phase for round $t$ terminates in at most $m+1 \leq 2\eta^*/\epsilon+2$ iterations. Crucially, every eliminated distribution $i$ satisfies the target guarantee $\err\del{f\tm{t};P_i} \leq \epsilon+\eta^*$. By a union bound, this result holds with probability at least 
\begin{align*}
    1-\delta'-\del{m+1}k\delta'' \geq 1-\delta'-2\del{\eta^*/\epsilon+1}k\delta'' \geq 1-\delta'-2\del{\eta/\epsilon+1}k\delta''
\end{align*}
The full procedure is detailed in \Cref{alg:MM-test}.

\ifcolt

\begin{algorithm}
\DontPrintSemicolon
\setcounter{AlgoLine}{0}
\caption{$\MMTest$}
\label{alg:MM-test}
\Input{Decision $f\tm{t}$, distribution set $\Uc\tm{t-1}$.}
\Output{$\Uc\tm{t}$.}
$\Uc\tm{t} \gets \Uc\tm{t-1}$; $\nu \gets 0$

\lFor{$i\in\Uc\tm{t-1}$}{
    $S_i\tm{t}\gets\emptyset$
}
\While{$\abs{\Uc\tm{t}} > \abs{\Uc\tm{t-1}}/2$}{
    \For{$i\in\Uc\tm{t}$}{
        $S_i\tm{t} \gets S_i\tm{t}\cup\tilde S_i\tm{t}$ where $\tilde S_i\tm{t}\iid P_i$ is of size $\abs{\tilde S_i\tm{t}} = T_\Tsf\del{ \nu,\epsilon/2,\frac{\delta}{4\del{\eta/\epsilon+1}kT} } - \abs{S_i\tm{t}}$

        \lIf{$\emperr\del{f\tm{t};S_i\tm{t}}\leq\epsilon/4+\nu$}{
            Remove $i$ from $\Uc\tm{t}$
        }
    }
    $\nu \gets \nu + \epsilon/2$
}
\end{algorithm}

\else

\begin{algorithm}
\caption{$\MMTest$}
\label{alg:MM-test}
\begin{algorithmic}[1]
    \Require Decision $f\tm{t}$, distribution set $\Uc\tm{t-1}$.
    \State $\Uc\tm{t} \gets \Uc\tm{t-1}$.
    \State $\nu \gets 0$.
    \For{$i\in\Uc\tm{t-1}$}
        \State $S_i\tm{t}\gets\emptyset$.
    \EndFor
    \While{$\abs{\Uc\tm{t}} > \abs{\Uc\tm{t-1}}/2$}
        \For{$i\in\Uc\tm{t}$}
            \State Sample $\tilde S_i\tm{t}\iid P_i$ of size $\abs{\tilde S_i\tm{t}} = T_\Tsf\del{ \nu,\epsilon/2,\frac{\delta}{4\del{\eta/\epsilon+1}kT} } - \abs{S_i\tm{t}}$.
            \State $S_i\tm{t} \gets S_i\tm{t}\cup\tilde S_i\tm{t}$.
            \If{$\emperr\del{f\tm{t};S_i\tm{t}}\leq\epsilon/4+\nu$}
                \State Remove $i$ from $\Uc\tm{t}$.
            \EndIf
        \EndFor
        \State $\nu \gets \nu + \epsilon/2$.
    \EndWhile
    \Ensure $\Uc\tm{t}$.
\end{algorithmic}
\end{algorithm}

\fi

\begin{theorem}[\eqref{eq:MDL-MM} upper bound]
\label{thm:MDL-MM-ub}
Let $T=\ceil*{\log_2 k}$, $\epsilon'=\frac{\epsilon}{32}$, $\delta'=\frac{\delta}{2T}$, $\Cond=\emptyset$ and $\Test = \MMTest$. Then, \Cref{alg:noisy-pers} solves \eqref{eq:MDL-MM} with sample complexity
\begin{align*}
    T_\MM\del{\eta_{1:k}^*, \eta_{1:k}, \epsilon, \delta} = O\del{ \log^2\del{\frac{k\del{\eta/\epsilon+1}}{\delta}} \del{ \frac{d\log\del{1/\epsilon}}{\epsilon\del{1-2\eta}} + \frac{k\del{\epsilon+\eta^*}}{\epsilon^2} } }
\end{align*}
\end{theorem}

\begin{remark}[Condition for separate learning]
In our algorithm for \eqref{eq:MDL-MM}, we do not explicitly enforce a switching condition for separate learning because the bound in \Cref{thm:MDL-MM-ub} depends on the unknown $\eta^*$. However, if we use the upper bound $\eta$ as a proxy for $\eta^*$ and define
\begin{align*}
    \Cond = \cbr{ \frac{d}{\epsilon\del{1-2\eta}} + \frac{k\del{\epsilon+\eta}}{\epsilon^2} > \sum_{i=1}^k \frac{d}{\epsilon\del{1-2\eta_i}} }
\end{align*}
then we achieve a sample complexity of $\tilde O\del{ \sum_{i=1}^k \frac{d}{\epsilon\del{1-2\eta_i}} }$ if $\Cond$ holds, and $\tilde O\del{ \frac{d}{\epsilon\del{1-2\eta}} + \frac{k\del{\epsilon+\eta^*}}{\epsilon^2} }$ otherwise.
\end{remark} \fi
\section{Structured Hypothesis Testing}
\label{sec:SHT}
As a precursor to establishing the MDL lower bound, we first investigate a related problem. Throughout this section, we fix the noise level to $\eta=1/4$ to focus our analysis on the other problem parameters.

Let $\Fc$ be a hypothesis class with VC-dimension $d$. Let $P$ be a distribution over $\Xc\times\cbr{0,1}$ such that the Bayes classifier $f^*$ belongs to $\Fc$ and satisfies the noise condition $P\del{f^*\del{X}\neq Y\middle\vert X=x} = 1/4$ for all $x\in\Xc$. This corresponds to the standard RCN setting with fixed, known noise. We are given a query function $f:\Xc\to\cbr{0,1}$ (not necessarily in $\Fc$) and tasked with determining whether $f$ is $\epsilon$-optimal with respect to $P$. Specifically, the algorithm must draw i.i.d. samples from $P$ and output a decision $\Dsf\in\cbr{\Ysf,\Nsf}$ that satisfies the following conditions with probability at least $1-\delta$:
\begin{equation}
\tag{SHT}\label{eq:SHT}
\begin{aligned}
    &\text{If $\err\del{f;P} \leq 1/4+\epsilon/12$, output $\Dsf=\Ysf$.} \\[6pt]
    &\text{If $\err\del{f;P} \geq 1/4+\epsilon$, output $\Dsf=\Nsf$.}
\end{aligned}
\end{equation}
We refer to this problem as \textit{Structured Hypothesis Testing (SHT)}. We say that an algorithm $\Ac$ for~\eqref{eq:SHT} has sample complexity $T_\Ac:\del{0,1}^2\to\Nb$ if, for any valid instance $\del{\Fc,P,f,\epsilon,\delta}$, it satisfies the conditions above using at most $T_\Ac\del{\epsilon,\delta}$ samples.

\ifcolt\else \subsection{\texorpdfstring{\eqref{eq:SHT}}{SHT} Upper Bounds}
In this section, we present two distinct strategies for solving the~\eqref{eq:SHT} problem.

\paragraph{Agnostic Testing.}
Since the goal is to test whether $P\del{f\del{X}\neq Y}$ is $\epsilon$-close to $1/4$, we can simply disregard the supervised nature of the data and work directly with the empirical errors $\del{\Ib\cbr{f\del{X_t}\neq Y_t}}_{t=1}^T$. The problem then reduces to distinguishing between two Bernoulli biases separated by a gap of $\Theta\del{\epsilon}$, which requires a sample complexity of $\Theta\del{\log\del{1/\delta}/\epsilon^2}$.

\begin{lemma}[Testing via empirical errors]
\label{lem:sht-test-emp-err}
Let $f:\Xc\to\cbr{0,1}$ be a fixed function, and let $S\iid P$ be a sample of size
\begin{align*}
    \abs{S} \geq T_\Csf\del{\epsilon,\delta} \coloneqq \frac{72\log\del{2/\delta}}{\epsilon^2}
\end{align*}
Then, with probability at least $1-\delta$, the following are true:
\begin{itemize}
    \item If $\emperr\del{f;S}\leq1/4+\epsilon/6$, then $\err\del{f;P}\leq1/4+\epsilon$.
    
    \item If $\emperr\del{f;S}>1/4+\epsilon/6$, then $\err\del{f;P}>1/4+\epsilon/12$.
\end{itemize}
\end{lemma}

\paragraph{Learning-Augmented Testing.}
An alternative approach is to first learn a reference hypothesis $\hat f$ that satisfies $\err\del{\hat f;P}\leq1/4+\epsilon$ with high probability. This can be achieved by performing ERM on a sample of size $O\del{d\log\del{1/\delta}/\epsilon}$. With $\hat f$ in hand, we can draw an additional sample to efficiently test the quality of a candidate function $f$. In the following lemma, we establish the validity of this approach for a general noise level $\eta$, as this broader result will be instrumental in later sections.

\begin{lemma}[From learning to testing]
\label{lem:sht-learn-to-test}
Let $P$ be a distribution over $\Xc\times\cbr{0,1}$ satisfying
\begin{align*}
    P\del{f^*\del{X}\neq Y\vert X=x} \leq \eta \quad\forall x\in\Xc
\end{align*}
for some $f^*:\Xc\to\cbr{0,1}$. Let $\eta^* \coloneqq P\del{f^*\del{X}\neq Y}$ denote the optimal error. Suppose we have access to a reference hypothesis $\hat f:\Xc\to\cbr{0,1}$ (possibly random) that satisfies
\begin{align*}
    \Pb\del{ \err\del{\hat f;P}\leq \eta^*+\frac{\epsilon}{48} } \geq 1-\frac{\delta}{3}
\end{align*}
Let $f:\Xc\to\cbr{0,1}$ be a fixed candidate function, and let $S\iid P$ be a sample drawn independently of $\hat f$, of size
\begin{align*}
    \abs{S} \geq T_\Lsf\del{\eta,\epsilon,\delta} \coloneqq \frac{96}{\epsilon\del{1-2\eta}} \log\del{\frac{6}{\delta}}
\end{align*}
Then, with probability at least $1-\delta$, the following are true:
\begin{itemize}
    \item If $\abs{\emperr\del{f;S}-\emperr\del{\hat f;S}}\leq\epsilon/3$, then $\err\del{f;P}\leq\eta^*+\epsilon$.
    
    \item If $\abs{\emperr\del{f;S}-\emperr\del{\hat f;S}}>\epsilon/3$, then $\err\del{f;P}>\eta^*+\epsilon/12$.
\end{itemize}
\end{lemma}

In the context of~\eqref{eq:SHT}, we apply \Cref{lem:sht-learn-to-test} with $\eta=\eta^*=1/4$. This yields a sample size of
\begin{align*}
    T_\Lsf\del{\frac{1}{4},\epsilon,\delta} = \frac{192}{\epsilon} \log\del{\frac{6}{\delta}}
\end{align*}
By combining the agnostic and learning-augmented strategies into a single procedure (\Cref{alg:SHT}), we obtain the following guarantee.

\ifcolt

\begin{algorithm}
\DontPrintSemicolon
\setcounter{AlgoLine}{0}
\caption{SHT}
\label{alg:SHT}
\Input{Function class $\Fc$, distribution $P$, function $f:\Xc\to\cbr{0,1}$, parameters $\del{\epsilon,\delta}\in\del{0,1}^2$.}
\Output{Decision $\Dsf$.}
\eIf(\Comment{Agnostic Testing.}){$d\geq1/\epsilon$}{
    Sample $S\iid P$ of size $\abs{S}=T_\Csf\del{\epsilon,\delta}$

    \leIf{$\emperr\del{f;S} \leq 1/4+\epsilon/6$}{$\Dsf \gets \Ysf$}{$\Dsf \gets \Nsf$}
}(\Comment{Learning-Augmented Testing.}){
    $\hat f \gets \ERM_\Fc\del{S}$ where $S\iid P$ is of size $\abs{S} = T_\SL\del{1/4, \epsilon/48,\delta/3}$

    Sample $S'\iid P$ of size $\abs{S'} = T_\Lsf\del{1/4,\epsilon,\delta}$

    \leIf{$\abs{\emperr\del{f;S'}-\emperr\del{\hat f;S'}}\leq\epsilon/3$}{$\Dsf \gets \Ysf$}{$\Dsf \gets \Nsf$}
}
\end{algorithm}

\else

\begin{algorithm}
\caption{SHT}
\label{alg:SHT}
\begin{algorithmic}[1]
    \Require Function class $\Fc$, distribution $P$, function $f:\Xc\to\cbr{0,1}$, parameters $\del{\epsilon,\delta}\in\del{0,1}^2$. 
    \If{$d\geq1/\epsilon$} \Comment{Agnostic Testing.}
        \State Sample $S\iid P$ of size $\abs{S}=T_\Csf\del{\epsilon,\delta}$.
        \If{$\emperr\del{f;S} \leq 1/4+\epsilon/6$}
            \State $\Dsf \gets \Ysf$.
        \Else
            \State $\Dsf \gets \Nsf$.
        \EndIf
    \Else \Comment{Learning-Augmented Testing.}
        \State Sample $S\iid P$ of size $\abs{S} = T_\SL\del{1/4, \epsilon/48,\delta/3}$.
        \State $\hat f \gets \ERM_\Fc\del{S}$.
        \State Sample $S'\iid P$ of size $\abs{S'} = T_\Lsf\del{1/4,\epsilon,\delta}$.
        \If{$\abs{\emperr\del{f;S'}-\emperr\del{\hat f;S'}}\leq\epsilon/3$}
            \State $\Dsf \gets \Ysf$.
        \Else
            \State $\Dsf \gets \Nsf$.
        \EndIf
    \EndIf
    \Ensure Decision $\Dsf$.
\end{algorithmic}
\end{algorithm}

\fi

\begin{theorem}[\eqref{eq:SHT} upper bound]
\label{thm:SHT-ub}
\Cref{alg:SHT} solves the \eqref{eq:SHT} problem with sample complexity
\begin{align*}
    T_\SHT\del{\epsilon,\delta} = O\del{ \log\del{\frac{1}{\delta}} \min\cbr{\frac{1}{\epsilon^2},\frac{d}{\epsilon}} }
\end{align*}
\end{theorem}

This bound highlights a natural trade-off: in the low-precision regime where $\epsilon \gg 1/d$, the agnostic testing approach ($1/\epsilon^2$) is superior. However, in the high-precision regime where $\epsilon \ll 1/d$, exploiting the supervised structure of the data via the learning-augmented approach ($d/\epsilon$) yields an improvement. \fi
\ifcolt\else \subsection{\texorpdfstring{\eqref{eq:SHT}}{SHT} Lower Bound}
Recall that under the RCN assumption, ERM requires only $\tilde O\del{d/\epsilon}$ samples, a significant improvement over the $\tilde O\del{d/\epsilon^2}$ complexity of the fully agnostic setting. It is natural to ask whether the testing problem~\eqref{eq:SHT} can similarly be solved with $\tilde O\del{1/\epsilon}$ samples.

The fast rate for ERM relies on a critical variance bound: for any distribution $P$ satisfying the RCN condition under constant noise, we have that (see \Cref{app:noisy-class})
\begin{align*}
    \Var_P\del{\Ib\cbr{f\del{X}\neq Y} - \Ib\cbr{f^*\del{X}\neq Y}} \lesssim \err\del{f;P} - \err\del{f^*;P}
\end{align*}
That is, we can control the variance when we \textit{offset} by $f^*$. Bernstein's inequality then ensures that with high probability,
\begin{align*}
    \abs{ \err\del{f;P}-\err\del{f^*;P}- \del{\emperr\del{f;S}-\emperr\del{f^*;S}} } \lesssim \frac{1}{\abs{S}} + \err\del{f;P}-\err\del{f^*;P}
\end{align*}
for a sample $S\iid P$. If we had access to the quantity $\abs{\emperr\del{f;S}-\emperr\del{f^*;S}}$, we could perform the test efficiently using the reasoning of \Cref{lem:sht-learn-to-test}. In fact, we showed that having access to an $\epsilon$-optimal hypothesis (a proxy for $f^*$) allows us to test with just $\tilde O\del{1/\epsilon}$ additional samples.

However, a fundamental obstacle remains: we do not know $f^*$. Although we know its population error is $\err\del{f^*;P}=1/4$, we cannot readily use its empirical error $\emperr\del{f^*;S}$ as a reference point because the concentration of $\abs{\err\del{f^*;P}-\emperr\del{f^*;S}}$ is slow, scaling with $\tilde O\del{1/\sqrt{\abs{S}}}$. Notably, the ERM sample size of $\tilde O\del{d/\epsilon}$ ensures uniform convergence of the excess risks (the offset counterparts), but not of the absolute errors $\abs{\err\del{f;P}-\emperr\del{f;S}}$.

In the following result, we prove that the trade-off achieved by \Cref{alg:SHT} is optimal. Specifically, we show that the sample complexity is lower-bounded by the minimum of the agnostic rate and the learning-augmented rate.

\begin{theorem}[\eqref{eq:SHT} lower bound]
\label{thm:SHT-lb}
Let $\delta\leq1/4$. Any algorithm $\Ac$ that solves~\eqref{eq:SHT} requires a sample size of
\begin{align*}
    T_\Ac\del{\epsilon,\delta} \geq \frac{3}{16}\log\del{1+\log 2} \min\cbr{ \frac{1}{\epsilon^2}, \frac{d}{\epsilon} } 
\end{align*}
\end{theorem}
\begin{proof}[Proof sketch of \Cref{thm:SHT-lb}]
We will show the bound $\Omega\del{d/\epsilon}$ for $d\leq\frac{1}{2\epsilon}$. When $d>\frac{1}{2\epsilon}$, we can simply choose a shattered set of size $\frac{1}{2\epsilon}$, so that the lower bound $\Omega\del{1/\epsilon^2}$ holds.

Hence, assume that $d\leq\frac{1}{2\epsilon}$ and suppose that $X$ is supported on $\Xc=\llbr{d^2}$. Let $p_X(0)=1-2d\epsilon$ and $p_X(x)=2\epsilon/d$ for each $x\in\sbr{d^2}$. Consider the following hypotheses:
\begin{equation*}
\begin{aligned}
    &\text{$H_0$: $f^*=f_0$, where $f_0\coloneqq x\mapsto 0$ is the constant zero function.} \\[6pt]
    &\text{$H_1$: $f^*$ equals $1$ precisely on a random subset $R\subset\sbr{d^2}$ of size $\abs{R}=d$.}
\end{aligned}
\end{equation*}
Note that we never need to shatter a set of size larger than $d$. Under $H_1$,
\begin{align*}
    \Pb\del{Y=1} = \del{1-2d\epsilon}\cdot\frac{1}{4} + \del{d^2-d}\cdot\frac{2\epsilon}{d}\cdot\frac{1}{4} + d\cdot\frac{2\epsilon}{d}\cdot\frac{3}{4} = \frac{1}{4}+\epsilon
\end{align*}
In particular, $f_0$ has errors $1/4$ under $H_0$ and $1/4+\epsilon$ under $H_1$, so that an~\eqref{eq:SHT} learner should be able to distinguish both hypotheses. Using Ingster's method, we can show that testing $H_0$ vs. $H_1$ under a uniform mixture over subsets $R$ requires $\Omega\del{d/\epsilon}$ samples.
\end{proof}

Combining \Cref{thm:SHT-ub} and~\cref{thm:SHT-lb}, we conclude that the agnostic and learning-augmented strategies fundamentally characterize the difficulty of the~\eqref{eq:SHT} problem, yielding a sample complexity of
\begin{align*}
    \tilde\Theta\del{ \min\cbr{\frac{1}{\epsilon^2},\frac{d}{\epsilon}} }
\end{align*} \fi
\ifcolt\else \subsection{Multi-Distribution SHT}
We now extend the~\eqref{eq:SHT} framework to the multi-distribution setting. Let $P_1,\dots,P_k$ be distributions over $\Xc\times\cbr{0,1}$ that all satisfy the RCN condition
\begin{align*}
    P_i\del{f^*\del{X}\neq Y\middle\vert X=x} = \frac{1}{4} \quad \forall x\in\Xc
\end{align*}
with respect to a \textit{shared} unknown $f^*\in\Fc$. We are given candidate functions $f_1,\dots,f_k:\Xc\to\cbr{0,1}$ (not necessarily in $\Fc$), and our objective is to determine whether $\err\del{f_i;P_i}$ is $\epsilon$-optimal for every $i\in\sbr{k}$. 

More precisely, the algorithm draws a total of $T$ samples (possibly adaptively) from the distributions and outputs a decision vector $\del{\Dsf_1,\dots,\Dsf_k} \in\cbr{\Ysf, \Nsf}^k$. We require that for each $i\in\sbr{k}$, with probability at least $1-\delta$, the following condition holds: 
\begin{equation}
\tag{MSHT}\label{eq:MSHT}
\begin{aligned}
    &\text{If $\err\del{f_i;P_i} \leq 1/4+\epsilon/12$, output $\Dsf_i=\Ysf$.} \\[6pt]
    &\text{If $\err\del{f_i;P_i} \geq 1/4+\epsilon$, output $\Dsf_i=\Nsf$.}
\end{aligned}
\end{equation}
We say that an algorithm $\Ac$ for~\eqref{eq:MSHT} has sample complexity $T_\Ac:\del{0,1}^2\to\Nb$ if it satisfies this guarantee for any valid instance $\del{\Fc,P_{1:k},f_{1:k},\epsilon,\delta}$ using at most $T_\Ac\del{\epsilon,\delta}$ total samples.

\begin{remark}[\eqref{eq:MSHT} high-probability guarantee]
Note that we only require that each $\Dsf_i$ succeeds with high-probability. That is, we do not require \textit{all} decisions to be correct simultaneously. The reason for this formulation is to avoid a $\log k$ factor when reducing to~\eqref{eq:MDL-RCN}. Of course, one can easily ensure simultaneous success by replacing $\delta$ with $\delta/k$ and taking a union bound over $i\in\sbr{k}$.
\end{remark}

\ifcolt

\begin{algorithm}[H]
\DontPrintSemicolon
\setcounter{AlgoLine}{0}
\caption{MSHT}
\label{alg:MSHT}
\Input{Function class $\Fc$, distributions $P_1,\dots,P_k$, functions $f_1,\dots,f_k:\Xc\to\cbr{0,1}$, parameters $\del{\epsilon,\delta}\in\del{0,1}^2$.}
\Output{Decision vector $\del{\Dsf_1,\dots,\Dsf_k}$.}
\eIf{$d\geq1/\epsilon$}{
    \For{$i=1,\dots,k$}{
        Sample $S_i\iid P_i$ of size $\abs{S_i}=T_\Csf\del{\epsilon,\delta}$

        \leIf{$\emperr\del{f_i;S_i} \leq 1/4+\epsilon/6$}{$\Dsf_i \gets \Ysf$}{$\Dsf_i \gets \Nsf$}
    }
}{
    \For{$i=1,\dots,k$}{
        $\hat f_i \gets \ERM_\Fc\del{S_i}$ where $S_i\iid P_i$ is of size $\abs{S_i} = T_\SL\del{1/4, \epsilon/48,\delta/3}$

        Sample $S_i'\iid P_i$ of size $\abs{S_i'} = T_\Lsf\del{1/4,\epsilon,\delta}$

        \leIf{$\abs{\emperr\del{f_i;S_i'}-\emperr\del{\hat f_i;S_i'}}\leq\epsilon/3$}{$\Dsf_i \gets \Ysf$}{$\Dsf_i \gets \Nsf$}
    }
}
\end{algorithm}

\else

\begin{algorithm}
\caption{MSHT}
\label{alg:MSHT}
\begin{algorithmic}[1]
    \Require Function class $\Fc$, distributions $P_1,\dots,P_k$, functions $f_1,\dots,f_k:\Xc\to\cbr{0,1}$, parameters $\del{\epsilon,\delta}\in\del{0,1}^2$.
    \If{$d\geq1/\epsilon$}
        \For{$i=1,\dots,k$}
            \State Sample $S_i\iid P_i$ of size $\abs{S_i}=T_\Csf\del{\epsilon,\delta}$.
            \If{$\emperr\del{f_i;S_i} \leq 1/4+\epsilon/6$}
                \State $\Dsf_i \gets \Ysf$.
            \Else
                \State $\Dsf_i \gets \Nsf$.
            \EndIf
        \EndFor
    \Else
        \For{$i=1,\dots,k$}
            \State Sample $S_i\iid P_i$ of size $\abs{S_i}= T_\SL\del{1/4, \epsilon/48,\delta/3}$.
            \State $\hat f_i \gets \ERM_\Fc\del{S_i}$.
            \State Sample $S_i'\iid P_i$ of size $\abs{S_i'} = T_\Lsf\del{1/4,\epsilon,\delta}$.
            \If{$\abs{\emperr\del{f_i;S_i'}-\emperr\del{\hat f_i;S_i'}}\leq\epsilon/3$}
                \State $\Dsf_i \gets \Ysf$.
            \Else
                \State $\Dsf_i \gets \Nsf$.
            \EndIf
        \EndFor
    \EndIf
    \Ensure Decision vector $\del{\Dsf_1,\dots,\Dsf_k}$.
\end{algorithmic}
\end{algorithm}

\fi

To solve the~\eqref{eq:MSHT} problem, we apply the strategy of \Cref{alg:SHT} to each distribution individually. This procedure is detailed in \Cref{alg:MSHT}, and its sample complexity is established in \Cref{thm:MSHT-ub}.

\begin{theorem}[\eqref{eq:MSHT} upper bound]
\label{thm:MSHT-ub}
\Cref{alg:MSHT} solves the \eqref{eq:MSHT} problem with sample complexity
\begin{align*}
    T_\MSHT\del{\epsilon,\delta} = O\del{ k\log\del{\frac{1}{\delta}} \min\cbr{\frac{1}{\epsilon^2},\frac{d}{\epsilon}} }
\end{align*}
\end{theorem}

Testing each distribution separately results in a $k$-fold increase in sample complexity relative to the single-distribution setting. We now show that this linear scaling is unavoidable. In particular, the following result establishes that in the worst case, no algorithm can solve~\eqref{eq:MSHT} without essentially performing independent tests.

\begin{theorem}[Multi-Distribution SHT lower bound]
\label{thm:MSHT-lb}
Let $\delta\leq0.01$ and $d\geq8\epsilon$. Any algorithm $\Ac$ that solves~\eqref{eq:MSHT} requires a sample size of
\begin{align*}
    T_\Ac\del{\epsilon,\delta} \geq \frac{0.015k}{4} \min\cbr{ \frac{1}{\epsilon^2}, \frac{d}{\epsilon} } 
\end{align*}
\end{theorem}
\begin{proof}[Proof sketch of \Cref{thm:MSHT-lb}]
Fix $\epsilon\in\del{0,1}$. Again, we focus on the setting where $d\leq\frac{1}{2\epsilon}$ and aim to show the lower bound $\Omega\del{kd/\epsilon}$. This immediately implies the $\Omega\del{k/\epsilon^2}$ bound when $d>\frac{1}{2\epsilon}$.

In essence, we will construct distributions whose $X$-marginals will be supported on consecutive disjoint sets of size $d^2+1$, each with the same weights as in the~\eqref{eq:SHT} lower bound. We then consider the following testing problem:
\begin{equation*}
\begin{aligned}
    &\text{$H_0$: $f^*=f_0$.} \\[6pt]
    &\text{$H_1$: $f^*$ equals $1$ precisely on a size-$d$ subset on the support of one of the distributions.}
\end{aligned}
\end{equation*}
which still only requires shattering a set of size $d$. We then show that the learner must sample $\Omega\del{d/\epsilon}$ times from each distribution by constructing an~\eqref{eq:SHT} algorithm that simulates the~\eqref{eq:MSHT} one.
\end{proof}

Combining the upper and lower bounds, we conclude that the sample complexity of the~\eqref{eq:MSHT} problem is
\begin{align*}
    \tilde\Theta\del{ k \cdot \min\cbr{\frac{1}{\epsilon^2},\frac{d}{\epsilon}} }
\end{align*} \fi
\section{MDL Lower Bound}
\label{sec:mdl-lb}
As in \Cref{sec:SHT}, we establish the lower bound under RCN noise rate $\eta_i=\eta_i^*=1/4$ for all $i\in\sbr{k}$. Note that in this setting, all variants of the MDL problem coincide.

Our argument relies on the observation that the testing problem~\eqref{eq:MSHT} is \textit{necessary} for MDL. That is, any algorithm that solves MDL can be readily converted into a solver for~\eqref{eq:MSHT}. Specifically, we can treat the MDL algorithm as a black box to learn proxies for the optimal classifier, and then use them to solve the testing problem with only $O\del{\log\del{1/\delta}/\epsilon}$ additional samples per distribution. We formally describe this reduction in \Cref{alg:MDL-MSHT} and derive its sample complexity in the following theorem.

\ifcolt
\else

\begin{algorithm}
\caption{MDL to MSHT}
\label{alg:MDL-MSHT}
\begin{algorithmic}[1]
    \Require Function class $\Fc$, distributions $P_1,\dots,P_k$, functions $f_1,\dots,f_k:\Xc\to\cbr{0,1}$, parameters $\del{\epsilon,\delta}\in\del{0,1}^2$, MDL algorithm $\Ac$ with sample complexity $T_\Ac$.
    \State $\del{\hat f_1,\dots,\hat f_k} \gets \Ac$ where we run $\Ac$ for $T_\Ac\del{\epsilon/48,\delta/3}$ rounds.
    \For{$i=1,\dots,k$}
        \State Sample $S_i\iid P_i$ of size $\abs{S_i} = T_\Lsf\del{1/4,\epsilon, \delta}$.
        \If{$\abs{\emperr\del{f_i;S_i}-\emperr\del{\hat f_i;S_i}}\leq\epsilon/3$}
            \State $\Dsf_i \gets \Ysf$.
        \Else
            \State $\Dsf_i \gets \Nsf$.
        \EndIf
    \EndFor
    \Ensure Decision vector $\del{\Dsf_1,\dots,\Dsf_k}$.
\end{algorithmic}
\end{algorithm}

\fi

\begin{lemma}[MDL to MSHT upper bound]
\label{lem:MDL-MSHT-ub}
\Cref{alg:MDL-MSHT}, under an MDL algorithm $\Ac$ with sample complexity $T_\Ac$, solves the \eqref{eq:MSHT} problem with sample complexity
\begin{align*}
    T_{\Ac\to\MSHT}\del{\epsilon,\delta} = T_\Ac\del{\frac{\epsilon}{48},\frac{\delta}{3}} + \frac{192k\log\del{6/\delta}}{\epsilon}
\end{align*}
\end{lemma}

Recall that our proposed strategies for~\eqref{eq:MDL-RCN} and~\eqref{eq:MDL-MM} achieve a sample complexity of $\tilde O\del{\min\cbr{dk/\epsilon, d/\epsilon+k/\epsilon^2}}$. We now show that this rate is optimal, up to logarithmic factors.

\begin{theorem}[MDL lower bound]
\label{thm:MDL-lb}
Let $\delta\leq0.01/3$, $d\geq384\epsilon$ and $\min\cbr{d,1/\epsilon}\geq 4\cdot10^7$. Any MDL algorithm $\Ac$ requires a sample size of
\begin{align*}
    T_\Ac\del{\epsilon,\delta} = \Omega\del{ \frac{d}{\epsilon} + k\cdot\min\cbr{\frac{1}{\epsilon^2}, \frac{d}{\epsilon}} }
\end{align*}
\end{theorem}
\section{\texorpdfstring{\eqref{eq:MDL-Mass}}{(MDL-Mass)} Lower Bound}
\label{sec:mdl-mass}

\ifcolt\else We now turn our attention to the most challenging MDL variant, \eqref{eq:MDL-Mass}. Recall that in \Cref{sec:mdl-lb}, we established a lower bound of $\Omega\del{d/\epsilon + k\cdot\min\cbr{1/\epsilon^2, d/\epsilon}}$. When $d \gtrsim 1/\epsilon$, the second term saturates at $k/\epsilon^2$, losing the multiplicative dependence on $d$.

In this section, we show that the stricter requirements of~\eqref{eq:MDL-Mass} imply a stronger lower bound that retains a multiplicative factor of $k\sqrt{d}$, even in the large-$d$ regime. To prove this, we follow the same recipe as in \Cref{sec:mdl-lb}: we introduce an auxiliary testing problem, denoted by~\eqref{eq:SHT-mass}, and demonstrate that any algorithm solving~\eqref{eq:MDL-Mass} must implicitly solve an instance of this testing problem for each distribution. \fi
\ifcolt\else \subsection{Massart Testing}
We begin by defining an auxiliary testing problem. Consider a pair of random variables $\del{X,Y}\in\llbr{d}\times\cbr{0,1}$, where the marginal distribution of $X$ is given by
\begin{align*}
    p_X\del{x} = \begin{dcases}
        1-\epsilon, & x=0 \\
        \epsilon/d, & x\in\sbr{d}
    \end{dcases}
\end{align*}
and the conditional distributions are $Y|X=x\sim\Ber\del{q_x}$. The vector of biases $\qbf = \del{q_x}_{x=0}^d$ is \textit{unknown}, but constrained such that $q_0\in\sbr{0,0.49}$ and $q_x\in\sbr{0,0.49}\cup\sbr{0.51,1}$ for all $x\in\sbr{d}$. Let
\begin{align*}
    \Delta_\qbf \coloneqq \frac{\epsilon}{d} \sum_{x\in\llbr{d}: q_x\geq1/2} \del{2q_x-1}
\end{align*}
The goal is to draw i.i.d. $\del{X,Y}$ pairs and determine which of the following two hypotheses holds:
\begin{equation}
\tag{SHT-Mass}\label{eq:SHT-mass}
\begin{aligned}
    &\text{$H_0$: If $q_x=0.465$ for all $x\in\llbr{d}$, output $\Ysf$ with probability at least $1-\delta$.} \\[6pt]
    &\text{$H_1$: If $\Delta_\qbf\geq0.3\epsilon$, output $\Nsf$ with probability at least $1-\delta$.}
\end{aligned}
\end{equation}

In the analysis below, we demonstrate that solving~\eqref{eq:SHT-mass} requires a polynomial dependence on $d$, even when $d \gg 1/\epsilon$, standing in sharp contrast to the behavior of~\eqref{eq:SHT}.

\begin{theorem}[SHT-Mass lower bound]
\label{thm:SHT-mass-lb}
Let $\delta\leq1/4$ and $d\geq 1750$. Any algorithm $\Ac$ that solves~\eqref{eq:SHT-mass} requires a sample size of
\begin{align*}
    T_\Ac\del{\epsilon,\delta} \geq \frac{\sqrt{d}}{2\sqrt{2}\epsilon}
\end{align*}
\end{theorem}

\begin{proof}[Proof sketch of \Cref{thm:SHT-mass-lb}]
Our ultimate goal is to construct bias vector $\qbf_h$, for each $h\in\cbr{0,1}$, that falls under hypothesis $H_h$ and such that $\qbf_0$ and $\qbf_1$ are difficult to distinguish. To start, we set the $0$th coordinate of each to $0.465$, so that the only informative samples are those in which $X\in\sbr{d}$. We then construct priors $\mu_h$ and sample the remaining biases i.i.d. from $\mu_h$ in a way that the resulting $\qbf_h$ is in $H_h$ with high probability. More importantly, we also make sure that the first moments of $\mu_0$ and $\mu_1$ match, which ensures that $Y$ has the same mean under both hypotheses. Since the marginal $p_X$ is independent of the hypothesis, it is thus necessary to exploit the relationship between $X$ and $Y$. This can only be accomplished once we observe repeated values of $X$ in $\sbr{d}$, which takes $\Omega\del{\sqrt{d}}$ samples. The $\Omega\del{\sqrt{d}/\epsilon}$ bound then follows from the fact that it takes $O\del{1/\epsilon}$ samples to obtain a single point in $\sbr{d}$.
\end{proof}

We emphasize that the moment matching argument requires varying the noise levels, preventing us from fixing the noise as in the~\eqref{eq:SHT} setup. Consequently, this construction is specific to~\eqref{eq:MDL-Mass} and does not extend to the other MDL variants.

\paragraph{Limitations of the~\eqref{eq:SHT-mass} construction.}
There is nothing inherently special about the $\sqrt{d}$ dependence. One could employ higher-order moment matching and polynomial methods to improve the exponent on $d$; see, e.g., \cite{wu-poly-test,CanonneTopicsDT2022}. In fact, with a more delicate analysis, our construction may plausibly yield a lower bound whose $d$-exponent exhibits a finer-grained dependence on the noise parameters. However, extending this all the way to a linear rate is impossible with the current~\eqref{eq:SHT-mass} construction.

To see why, observe that to obtain a constant $\Delta_\qbf$ gap, a constant fraction of the coordinates must be heavy (i.e., $q_x\geq0.51$) under $H_1$. This creates a vulnerability: an algorithm need only detect \textit{one} such heavy coordinate to reject $H_0$. Let $C=O\del{1}$ be the number of samples required to reliably distinguish between $\Ber\del{0.49}$ and $\Ber\del{0.51}$. Consider a learner that samples until it observes $C$-collisions on a small, constant number of unique coordinates. Under $H_1$, at least one of these coordinates will be heavy with high probability, allowing the learner to identify it. Since finding $C$-collisions requires only $O\del{d^{1-1/C}}$ samples, the hardness of this instance is strictly sublinear in $d$.

We note that while this does not rule out the existence of other hard instances achieving a linear dependence, constructing one that yields an MDL lower bound via a similar reduction is nontrivial. This is because our choice of $H_0$ (containing only light biases) is critical for ensuring that the resulting hypothesis class maintains a fixed VC dimension when extended to multiple distributions. \fi
\ifcolt\else \subsection{Lower Bound}
Next, we construct an MDL instance that inherits the~\eqref{eq:SHT-mass} difficulty. Let the covariate space be $\Xc\coloneqq\llbr{kd+k-1}$, which we partition into $k$ disjoint blocks of size $d+1$. We denote the $i$th block by $\Xc_i \coloneqq \cbr{\del{i-1}\del{d+1}+x: x\in\llbr{d}}$ for each $i\in\sbr{k}$. 

We define the hypothesis class $\Fc$ to be the set of binary functions consisting of:
\begin{itemize}
    \item The zero function $f_0$.
    
    \item All functions that equal $1$ precisely on a subset of $\Xc_i\backslash\cbr{\del{i-1}\del{d+1}}$ (i.e., we exclude the first coordinate) for some $i\in\sbr{k}$.
\end{itemize}
Note that $\VC\del{\Fc}=d$. 

Next, we define distributions $P_1,\dots,P_k$ over $\Xc\times\cbr{0,1}$ with marginal PMFs
\begin{align*}
    P_i\del{X=x} = \begin{dcases}
        1-\epsilon, & x = \del{i-1}\del{d+1} \\
        \epsilon/d, & x \in \cbr{ \del{i-1}\del{d+1}+1,\dots,i\del{d+1}-1 }
    \end{dcases}
\end{align*}
That is, the probability mass is distributed as $\del{1-\epsilon,\epsilon/d,\dots,\epsilon/d}$ on the corresponding $\Xc_i$. Furthermore, we require the label noise to satisfy the Massart constraint with respect to some target $f^*\in\Fc$:
\begin{align*}
    \eta_i^*\del{x} \coloneqq P_i\del{f^*\del{X}\neq Y\vert X=x} \leq 0.49 \quad\forall x\in\Xc, i\in\sbr{k}
\end{align*}
Recall that while the optimal risk values $\eta_i^* = \err\del{f^*;P_i}$ are unknown, we know the noise upper bounds $\eta_1=\dots=\eta_k=0.49$. The goal of~\eqref{eq:MDL-Mass} is to sample from the distributions and output hypotheses $\hat f_1,\dots,\hat f_k$ satisfying
\begin{align*}
    \Pb\del{ \max_{i\in\sbr{k}}\cbr{ \err\del{\hat f_i;P_i} - \eta_i^* } \leq \epsilon } \geq 1-\delta
\end{align*}
The following theorem establishes that any algorithm solving this problem must incur the multiplicative cost discussed.

\begin{theorem}[\eqref{eq:MDL-Mass} lower bound]
\label{thm:MDL-Mass-lb}
Let $\delta\leq0.1/3$ and $d\geq7\cdot10^{10}$. Any algorithm $\Ac$ that solves~\eqref{eq:MDL-Mass} in this setup requires a sample size of
\begin{align*}
    T_\Ac\del{\epsilon,\delta} = \Omega\del{\frac{k\sqrt{d}}{\epsilon}}
\end{align*}
\end{theorem}

By combining this result with the general lower bound from \Cref{thm:MDL-lb}, we can conclude that, under a constant noise upper bound of $\eta_i\leq0.49$ for every $i\in\sbr{k}$, a successful algorithm for~\eqref{eq:MDL-Mass} requires a sample size of at least
\begin{align*}
    \Omega\del{\frac{d}{\epsilon} + k\cdot\min\cbr{\frac{1}{\epsilon^2},\frac{d}{\epsilon}} + \frac{k\sqrt{d}}{\epsilon}}
\end{align*} \fi
\section{Discussion}
In this work, we investigated the limits of multi-distribution learning under bounded label noise, asking whether the fast rates achievable in the single-distribution setting extend to the multi-distribution regime. We formulated three distinct variants of the problem based on the benchmarks they target: the known noise upper bounds, the minimax error, and the true unknown optimal errors.

For the first two variants, we devised strategies that achieve sample complexities $\tilde O\del{ \frac{d}{\epsilon\del{1-2\eta}}+\sum_{i=1}^k \frac{\epsilon+\eta_i}{\epsilon^2} }$ and $\tilde O\del{ \frac{d}{\epsilon\del{1-2\eta}}+\frac{k\del{\epsilon+\eta^*}}{\epsilon^2} }$. Through a reduction to multi-distribution hypothesis testing, we proved that these are optimal even under constant noise, ruling out a fast rate of $\tilde O\del{\frac{d+k}{\epsilon}}$. This highlights a fundamental departure from the realizable and agnostic MDL complexities, where the overhead for scaling to $k$ distributions is merely additive.

Finally, we showed that the third variant is strictly harder by constructing a lower bound of $\Omega\del{k\sqrt{d}/\epsilon}$ via moment matching. This establishes a statistical separation between RCN and Massart noise in the MDL setting and reveals that for this objective, the penalty for learning multiple distributions is multiplicative rather than additive.

\paragraph{Future Directions.}
Our work leaves several compelling avenues for future research.
\begin{enumerate}
\item\textbf{Centralized MDL:} We focused on personalized learning, where the algorithm outputs a distinct hypothesis $\hat f_i$ for each distribution. It remains unclear whether similar rates apply to the centralized setting~\cite{nika-on-demand}, where the learner is constrained to produce a single common hypothesis $\hat f$ with small error across all distributions simultaneously.

\item\textbf{Noise-dependent Lower Bounds:} While our lower bounds establish optimality in the constant-noise regime, our analysis does not fully characterize the complexity as a function of the noise rate.

\item\textbf{Closing the Gap for~\eqref{eq:MDL-Mass}:} A tight sample complexity bound for the hardest variant, \eqref{eq:MDL-Mass}, remains unknown. While we established a lower bound of 
\begin{align*}
    \Omega\del{\frac{d}{\epsilon} + k\cdot\min\cbr{\frac{1}{\epsilon^2},\frac{d}{\epsilon}} + \frac{k\sqrt{d}}{\epsilon}}
\end{align*}
it is unclear whether the true complexity scales linearly as $\Omega\del{kd/\epsilon}$. Determining the precise interplay between $d$ and $k$ in this setting likely requires a fundamentally different construction, as discussed in \Cref{sec:mdl-mass}.
\end{enumerate}

\section*{Acknowledgments}
We are grateful to Eric Zhao, Nika Haghtalab, Moïse Blanchard and Alexander Rakhlin for insightful discussions.

\newpage
\bibliographystyle{alpha}
\bibliography{references}
\newpage

\appendix
\section{Concentration Inequalities}
We record a few standard concentration inequalities that will be used throughout. For background, we refer to~\cite{Boucheron2016-ju}.

\begin{lemma}[Concentration inequalities]
Let $Z_1,\dots,Z_T$ be i.i.d. random variables and define the empirical mean $\bar Z\coloneqq \frac{1}{T}\sum_{t=1}^T Z_t$.
\begin{itemize}
    \item (Hoeffding) If $Z_1\in\sbr{a,b}$, then
    \begin{align*}
        \Pb\del{ \abs{ \bar Z - \Eb\sbr{Z_1} } \leq \del{b-a} \sqrt{\frac{1}{2T}\log\del{\frac{2}{\delta}}} } \geq 1-\delta
    \end{align*}

    \item (Bernstein) If $\abs{Z_1-\Eb\sbr{Z_1}}\leq b$, then
    \begin{align*}
        \Pb\del{ \abs{ \bar Z - \Eb\sbr{Z_1} } \leq \frac{2b}{3T}\log\del{\frac{2}{\delta}} + \sqrt{\frac{2\Var\del{Z_1}}{T}\log\del{\frac{2}{\delta}}} } \geq 1-\delta
    \end{align*}
\end{itemize}
\end{lemma}
\section{Learning under Massart Noise}
\label{app:noisy-class}
Let $P$ be a distribution over $\del{X,Y}\in\Xc\times\cbr{0,1}$ such that
\begin{align*}
    \eta^*\del{x} \coloneqq P\del{f^*\del{X}\neq Y\vert X=x} \leq \eta < \frac{1}{2} \quad\forall x\in\Xc
\end{align*}
for some $f^*:\Xc\to\cbr{0,1}$. This is the Massart noise condition with rate $\eta$. We denote the Bayes error by $\eta^* \coloneqq P\del{f^*\del{X}\neq Y}$. The following result characterizes the pointwise error of a classifier in terms of $f^*$.

\begin{lemma}
\label{lem:noisy-class-exc-err}
For any $f:\Xc\to\cbr{0,1}$,
\begin{align*}
    P\del{f\del{X}\neq Y\vert X=x} - \eta^*\del{x} = \del{1-2\eta^*\del{x}} \Ib\cbr{f\del{x}\neq f^*\del{x}}
\end{align*}
\end{lemma}
\begin{proof}[Proof of \Cref{lem:noisy-class-exc-err}]
This follows from
\begin{align*}
    P\del{f\del{X}\neq Y\vert X=x} &= \eta^*\del{x}\Ib\cbr{f\del{x}=f^*\del{x}} + \del{1-\eta^*\del{x}} \Ib\cbr{f\del{x}\neq f^*\del{x}} \\
    &= \eta^*\del{x}\del{1 - \Ib\cbr{f\del{x}\neq f^*\del{x}}} + \del{1-\eta^*\del{x}} \Ib\cbr{f\del{x}\neq f^*\del{x}} \\
    &= \eta^*\del{x} + \del{1-2\eta^*\del{x}} \Ib\cbr{f\del{x}\neq f^*\del{x}}
\end{align*}
\end{proof}

The bounded noise assumption implies that $1-2\eta^*\del{x}\geq1-2\eta$ for any $x\in\Xc$. Then, if we integrate the equality of \Cref{lem:noisy-class-exc-err} with respect to marginal $P_X$ and rearrange, we obtain
\begin{align*}
    P\del{f\del{X}\neq f^*\del{X}} \leq \frac{P\del{f\del{X}\neq Y} - \eta^*}{1-2\eta}
\end{align*}
The key idea behind the Massart fast rate is that we can bound the following variance:
\begin{align*}
    \Var_P\del{\Ib\cbr{f\del{X}\neq Y} - \Ib\cbr{f^*\del{X}\neq Y}} &\leq \Eb_P\sbr{\abs{\Ib\cbr{f\del{X}\neq Y} - \Ib\cbr{f^*\del{X}\neq Y}}} \\
    &= P\del{f\del{X}\neq f^*\del{X}} \\
    &\leq \frac{P\del{f\del{X}\neq Y} - \eta^*}{1-2\eta}
\end{align*}
Suppose that we are trying to learn a finite function class $\Fc$ such that $f^*\in\Fc$. Let $\del{X_t,Y_t}_{t=1}^T \iid P$ be a sample and consider the empirical minimizer
\begin{align*}
    \hat f \in \argmin_{f\in\Fc} \cbr{ \frac{1}{T} \sum_{t=1}^T \Ib\cbr{f\del{X_t}\neq Y_t} }
\end{align*}
Applying Bernstein's inequality on random variables $\Ib\cbr{f\del{X_t}\neq Y_t} - \Ib\cbr{f^*\del{X_t}\neq Y_t}$ and taking a union bound over $\Fc$, we can conclude that
\begin{align*}
    P\del{\hat f\del{X}\neq Y} - \eta^* &\lesssim \frac{\log\del{\abs{\Fc}/\delta}}{n\del{1-2\eta}}
\end{align*}
In other words, $\hat f$ is $\epsilon$-optimal provided that $n\gtrsim \frac{\log\del{\abs{\Fc}/\delta}}{\epsilon\del{1-2\eta}}$. To extend this idea to general VC classes, we refer to~\cite{Boucheron2005-il}.
\section{Proofs of \Cref{sec:ub}}

\subsection{Proof of \Cref{lem:sl-opt}}
Note that the mixture also satisfies the bounded noise assumption:
\begin{align*}
    \bar P_{\Uc}\del{ f^*\del{X}\neq Y\middle\vert X=x } &= \sum_{i\in\Uc} \alpha_i\del{x} P_i\del{ f^*\del{X}\neq Y\middle\vert X=x } \leq \eta
\end{align*}
where $\alpha_i \coloneqq \frac{p_i^X}{\sum_j p_j^X}$ and $p_i^X$ is the density of the $X$-marginal of $P_i$ with respect to some dominating measure. Then, the ERM solution $\hat f$ on $T_\SL\del{\eta,\epsilon,\delta}$ samples from $\bar P_{\Uc}$ ensures that 
\begin{align*}
    \bar P_{\Uc}\del{ \hat f\del{X}\neq Y } \leq \epsilon + \bar P_{\Uc}\del{ f^*\del{X}\neq Y } = \epsilon + \frac{1}{\abs{\Uc}}\sum_{i\in\Uc} \eta_i^*
\end{align*}
with probability $\geq1-\delta$.
\subsection{Proof of \Cref{lem:test-comp}}
For convenience, let us rewrite the sample size lower bound as $\abs{S} \geq 2\log\del{\frac{2}{\delta}} \frac{\epsilon/8+\nu}{\del{\epsilon/8}^2}$. By Bernstein's inequality, we know that
\begin{align*}
    \abs{\err\del{f;P}-\emperr\del{f;S}} \leq \sqrt{\frac{2 \err\del{f;P} \log\del{2/\delta}}{\abs{S}}} + \frac{2\log\del{2/\delta}}{\abs{S}} \leq \frac{\epsilon}{8} \sqrt{\frac{\err\del{f;P}}{\epsilon/8+\nu}} + \frac{\epsilon}{8}
\end{align*}
with probability $1-\delta$. Suppose that $\err\del{f;P}\leq\epsilon/8+\nu$. Then,
\begin{align*}
    \emperr\del{f;S} &\leq \err\del{f;P} + \frac{\epsilon}{8} \sqrt{\frac{\err\del{f;P}}{\epsilon/8+\nu}} + \frac{\epsilon}{8} \leq \frac{\epsilon}{2} + \nu
\end{align*}
Now, suppose that $\err\del{f;P}>\epsilon+\nu$ and define $g\del{p}\coloneqq p-a\sqrt{\frac{p}{a+b}}$. Then $g'\del{p} = 1 - \frac{a}{2\sqrt{p\del{a+b}}}$. In other words, $g$ is increasing for $a\leq2\sqrt{p\del{a+b}} \iff p\geq\frac{a^2}{4\del{a+b}}$. Specializing to $p=\epsilon+\nu$, $a=\epsilon/8$ and $b=\nu$, we see that this condition is clearly ensured:
\begin{align*}
    \epsilon+\nu \geq \frac{\epsilon}{8} > \frac{\del{\epsilon/8}^2}{4\del{\epsilon/8+\nu}}
\end{align*}
Since $\err\del{f;P}>\epsilon+\nu$ by assumption, we can again apply Bernstein to obtain
\begin{align*}
    \emperr\del{f;S} &\geq \err\del{f;P} - \frac{\epsilon}{8} \sqrt{\frac{\err\del{f;P}}{\epsilon/8+\nu}} - \frac{\epsilon}{8} \\
    &> \epsilon + \nu - \frac{\epsilon}{8} \sqrt{\frac{\epsilon+\nu}{\epsilon/8+\nu}} - \frac{\epsilon}{8} \\
    &> \frac{\epsilon}{2} + \nu
\end{align*}
where we used the fact that $\sqrt{\frac{\epsilon+\nu}{\epsilon/8+\nu}}\leq\sqrt{8}$ and $7/8-1/\sqrt{8}>1/2$.
\subsection{Proof of \Cref{thm:MDL-RCN-ub}}
By the reasoning above, we know that we succeed with probability $1-\delta$ under the specified parameters. Since we learn at least half of the distributions on every round, we only require $T=\ceil*{\log_2 k}$ rounds. On every round, ERM requires $T_\SL\del{\eta,\frac{\epsilon}{16}, \frac{\delta}{2T}}$ samples. For testing, each $P_i$ is sampled at most $T_\Tsf\del{\eta_i,\epsilon,\frac{\delta}{2kT}}$ times on every round. Hence, the total sample size required is
\begin{align*}
    T_\RCN\del{\eta_{1:k}, \epsilon, \delta} &\leq T\del{ T_\SL\del{\eta, \frac{\epsilon}{16}, \frac{\delta}{2T}} + \sum_{i=1}^k T_\Tsf\del{\eta_i,\epsilon,\frac{\delta}{2kT}} } \\
    &\lesssim \del{\log k} \del{ \frac{d\log\del{1/\epsilon} + \log\del{\frac{\log k}{\delta}}}{\epsilon\del{1-2\eta}} + \log\del{\frac{k}{\delta}} \sum_{i=1}^k \frac{\epsilon+\eta_i}{\epsilon^2} } \\
    &\leq \log^2\del{\frac{k}{\delta}} \del{ \frac{d\log\del{1/\epsilon}}{\epsilon\del{1-2\eta}} + \sum_{i=1}^k \frac{\epsilon+\eta_i}{\epsilon^2} }
\end{align*}
\ifcolt\else \subsection{Proof of \Cref{thm:MDL-MM-ub}}
From the description of the procedure, and by tuning $\delta'$ and $\delta''$ appropriately, we know that we can guarantee success with probability at least $1-\delta$. Again, we learn half of the distributions on every round, so that $T=\ceil*{\log_2 k}$ rounds suffices. On each round, ERM uses
\begin{align*}
    T_\SL\del{\eta,\frac{\epsilon}{32}, \frac{\delta}{2T}} = O\del{ \frac{d + \log\del{\frac{\log k}{\delta}}}{\epsilon\del{1-2\eta}} }
\end{align*}
samples, and testing requires at most 
\begin{align*}
    T_\Tsf\del{ \eta^*+\frac{\epsilon}{2}, \frac{\epsilon}{2}, \frac{\delta}{4\del{\eta/\epsilon+1}kT} } = 32\log\del{\frac{8\del{\eta/\epsilon+1}kT}{\delta}} \frac{9\epsilon+16\eta^*}{\epsilon^2}
\end{align*}
samples per distribution. Putting these together, we obtain a total sample complexity upper bound of
\begin{align*}
    T_\MM\del{\eta_{1:k}^*, \eta_{1:k}, \epsilon, \delta} &\lesssim \del{\log k}\del{ \frac{d + \log\del{\frac{\log k}{\delta}}}{\epsilon\del{1-2\eta}} + \log\del{\frac{\del{\eta/\epsilon+1}k}{\delta}} \frac{ k\del{\epsilon+\eta^*} }{\epsilon^2} } \\
    &\lesssim \log^2\del{\frac{\del{\eta/\epsilon+1}k}{\delta}} \del{ \frac{d}{\epsilon\del{1-2\eta}} + \frac{k\del{\epsilon+\eta^*}}{\epsilon^2} }
\end{align*}
 \fi
\section{Proofs of \Cref{sec:SHT}}

\subsection{Proof of \Cref{lem:sht-test-emp-err}}
Hoeffding's inequality tells us that, with probability at least $1-\delta$,
\begin{align*}
    \abs{\err\del{f;P}-\emperr\del{f;S}} \leq \sqrt{\frac{\log\del{2/\delta}}{2\abs{S}}} \leq \frac{\epsilon}{12}
\end{align*}
Under this high-probability event, we immediately obtain our claim:
\begin{align*}
    \emperr\del{f;S} \leq \frac{1}{4}+\frac{\epsilon}{6} &\Longrightarrow \err\del{f;P} \leq \frac{1}{4}+\frac{\epsilon}{6}+\frac{\epsilon}{12} \leq \frac{1}{4}+\epsilon \\
    \emperr\del{f;S} > \frac{1}{4}+\frac{\epsilon}{6} &\Longrightarrow \err\del{f;P} > \frac{1}{4}+\frac{\epsilon}{6}-\frac{\epsilon}{12} = \frac{1}{4}+\frac{\epsilon}{12} 
\end{align*}
\subsection{Proof of \Cref{lem:sht-learn-to-test}}
From \Cref{app:noisy-class}, we know that
\begin{align*}
    \Var_P\del{ \Ib\cbr{g\del{X}\neq Y} - \Ib\cbr{f^*\del{X}\neq Y} } \leq \frac{\err\del{g;P}-\eta^*}{1-2\eta}
\end{align*}
for any function $g:\Xc\to\cbr{0,1}$. Bernstein's inequality then tells us that, with probability $\geq1-\delta/3$,
\begin{align*}
    \abs{ \err\del{g;P}-\eta^* - \del{\emperr\del{g;S}-\emperr\del{f^*;S}} } &\leq \frac{4}{3\abs{S}} \log\del{\frac{6}{\delta}} + \sqrt{ \frac{2\del{\err\del{g;P}-\eta^*}}{\del{1-2\eta}\abs{S}} \log\del{\frac{6}{\delta}} } \\
    &\leq \frac{\epsilon}{48} + \sqrt{ \frac{\epsilon}{48} \del{ \err\del{g;P}-\eta^* } } \\
    &\leq \frac{3\epsilon}{96} + \frac{\err\del{g;P}-\eta^*}{2}
\end{align*}
where we used AM-GM in the last inequality. Since $\hat f$ and $S$ are independent, we can then conclude that $\hat f$ and $f^*$ are close in empirical error: with probability $\geq1-2\delta/3$,
\begin{align*}
    \abs{ \emperr\del{\hat f;S}-\emperr\del{f^*;S} } \leq \frac{3\epsilon}{96} + \frac{3}{2} \del{\err\del{\hat f;P} - \eta^*} \leq \frac{3\epsilon}{48}
\end{align*}
Let us take a union bound of this event with the Bernstein bound above on our function of interest $f$, so that they simultaneously occur with probability $\geq1-\delta$. Under this event, we then have that
\begin{align*}
    &\abs{\emperr\del{f;S}-\emperr\del{\hat f;S}}\leq\frac{\epsilon}{3} \\
    &\hspace{1cm} \Longrightarrow \abs{\emperr\del{f;S}-\emperr\del{f^*;S}} \leq \abs{\emperr\del{f;S}-\emperr\del{\hat f;S}} + \abs{\emperr\del{\hat f;S}-\emperr\del{f^*;S}} \leq \frac{19\epsilon}{48} \\
    &\hspace{1cm} \Longrightarrow \err\del{f;P} - \eta^* \leq \frac{3\epsilon}{48} + 2\abs{\emperr\del{f;S}-\emperr\del{f^*;S}} \leq \frac{41\epsilon}{48}
\end{align*}
and
\begin{align*}
    &\err\del{f;P}-\eta^*\leq\frac{\epsilon}{12} \\
    &\hspace{1cm} \Longrightarrow \abs{\emperr\del{f;S}-\emperr\del{f^*;S}} \leq \frac{3\epsilon}{96} + \frac{3}{2}\del{\err\del{f;P}-\eta^*} \leq \frac{5\epsilon}{32} \\
    &\hspace{1cm} \Longrightarrow \abs{\emperr\del{f;S}-\emperr\del{\hat f;S}} \leq \abs{\emperr\del{f;S}-\emperr\del{f^*;S}} + \abs{\emperr\del{f^*;S}-\emperr\del{\hat f;S}} \leq \frac{21\epsilon}{96}
\end{align*}
\subsection{Proof of \Cref{thm:SHT-ub}}
When $d\geq1/\epsilon$, \Cref{lem:sht-test-emp-err} immediately implies correctness of the decision, and the sample size is
\begin{align*}
    T_\Csf\del{\epsilon,\delta} = O\del{\frac{\log\del{1/\delta}}{\epsilon^2}}
\end{align*}

Now, suppose that $d<1/\epsilon$. With $T_\SL\del{1/4,\epsilon/48,\delta/3}$ samples, we know that ERM guarantees $\err\del{\hat f;P}\leq1/4+\epsilon/48$ with probability at least $1-\delta/3$. Then, correctness of our decision follows from \Cref{lem:sht-learn-to-test}. In this case, the sample size obtained is
\begin{align*}
    T_\SL\del{ \frac{1}{4}, \frac{\epsilon}{48}, \frac{\delta}{3} } + T_\Lsf\del{\frac{1}{4},\epsilon,\delta} = O\del{ \frac{d}{\epsilon} \log\del{\frac{1}{\delta}} }
\end{align*}
Combining both conditions yields the desired sample complexity.
\subsection{Proof of \Cref{thm:SHT-lb}}
To prove the hardness of~\eqref{eq:SHT}, we first construct a preliminary base testing problem and establish a lower bound for it.

\begin{theorem}
\label{thm:SHT-lb-prelim}
Let $d\in\Nb$ and $\epsilon\in\del{0,1}$ be such that $d\epsilon\leq1$. Define a random variable $X\in\llbr{d^2}$ with PMF
\begin{align*}
    p_X\del{x} = \begin{dcases}
        1-d\epsilon, & x=0 \\
        \epsilon/d, & x\in\sbr{d^2}
    \end{dcases}
\end{align*}
In addition, let $f^*:\llbr{d^2}\to\cbr{0,1}$ be an underlying function, and $Y\in\cbr{0,1}$ be such that
\begin{align*}
    \Pb\del{f^*\del{X}\neq Y\vert X=x} = 1/4 \quad\forall x\in\llbr{d^2}
\end{align*}
We sample i.i.d. pairs $\del{X,Y}$ and must ensure that:
\begin{equation}
\tag{SHT-Base}\label{eq:SHT-Base}
\begin{aligned}
    &\text{$H_0$: If $f^*=f_0$, output $\Ysf$ with probability at least $3/4$.} \\[6pt]
    &\text{$H_1$: If $f^*$ equals $1$ precisely on a size-$d$ subset, output $\Nsf$ with probability at least $3/4$.}
\end{aligned}
\end{equation}
This problem requires a sample size of
\begin{align*}
    T \geq \frac{3d}{4\epsilon}\log\del{1+\log 2}
\end{align*}
\end{theorem}

\begin{proof}[Proof of \Cref{thm:SHT-lb-prelim}]
Let $P_0^T$ be the joint distribution over $\del{X_t,Y_t}_{t=1}^T$ under $H_0$ and, similarly, $P_R^T$ the joint under $H_1$ and subset $R\subset\sbr{d^2}$ where $f^*$ equals $1$. We also use lower-case $p$ to denote the corresponding PMFs. With abuse of notation, let $\Ysf$ and $\Nsf$ denote the event that a successful algorithm outputs each decision. Let $R$ be chosen uniformly at random from all subsets of size $d$. The learner must ensure that
\begin{align*}
    \frac{3}{2} \leq P_0^T\del{\Ysf} + \Eb_R\sbr{P_R^T\del{\Nsf}} \leq 1 + \TV\del{P_0^T,\Eb_R\sbr{P_R^T}}
\end{align*}
In other words, we need 
\begin{align*}
    \TV\del{P_0^T,\Eb_R\sbr{P_R^T}} \geq \frac{1}{2}
\end{align*}
To upper bound the TV distance, we will work with the $\chi^2$ divergence and apply Ingster's method \cite{Ingster2002-em,Polyanskiy2025-uh}:
\begin{align*}
    \chi^2\del{\Eb_R\sbr{P_R^T}\middle\Vert P_0^T} + 1 = \Eb_{R,R'}\sbr{ G^T\del{R,R'} }
\end{align*}
where $R$ and $R'$ are i.i.d. subsets and 
\begin{align*}
    G\del{R,R'} &\coloneqq \Eb_{\del{X,Y}\sim P_0}\sbr{ \frac{p_R\del{X,Y}}{p_0\del{X,Y}}\cdot \frac{p_{R'}\del{X,Y}}{p_0\del{X,Y}} } \\
    &= \sum_{x\in\llbr{d^2}} p_0\del{x} \sum_{y\in\cbr{0,1}} \frac{p_R\del{y|x} p_{R'}\del{y|x}}{p_0\del{y|x}} \\
    &= 1-d\epsilon + \frac{\epsilon}{d} \sum_{x\in\sbr{d^2}} \underbrace{ \sbr{ \frac{4}{3}p_R\del{0|x}p_{R'}\del{0|x} + 4p_R\del{1|x}p_{R'}\del{1|x} } }_{\eqqcolon\del{\star}}
\end{align*}
To evaluate this sum, let us consider two possibilities:
\begin{itemize}
    \item For $x\notin R\cap R'$, at least one of the conditional PMFs will be $\Ber\del{1/4}$. Let the other one be $\Ber\del{q}$. Then,
    \begin{align*}
        \del{\star} = \frac{4}{3}\cdot\frac{3}{4}\cdot \del{1-p} + 4\cdot\frac{1}{4}\cdot p = 1
    \end{align*}

    \item For $x\in R\cap R'$,
    \begin{align*}
        \del{\star} = \frac{4}{3}\cdot\del{\frac{1}{4}}^2 + 4\cdot\del{\frac{3}{4}}^2 = \frac{7}{3}
    \end{align*}
\end{itemize}
As a result, we obtain
\begin{align*}
    G\del{R,R'} = 1-d\epsilon + \frac{\epsilon}{d}\del{ d^2 - \abs{R\cap R'} + \frac{7}{3}\abs{R\cap R'} } = 1 + \frac{4\epsilon}{3d}\abs{R\cap R'}
\end{align*}
Plugging this back into the $\chi^2$ formula yields
\begin{align*}
    \chi^2\del{\Eb_R\sbr{P_R^T}\middle\Vert P_0^T} + 1 &= \Eb_{R,R'}\sbr{ \del{1 + \frac{4\epsilon}{3d}\abs{R\cap R'}}^T } \leq \Eb_{R,R'}\sbr{ \exp\del{\frac{4\epsilon T}{3d}\abs{R\cap R'}} }
\end{align*}
Next, we observe that the random variable $\abs{R\cap R'}$ follows a hypergeometric distribution: conditioned on $R$, we can see $\abs{R\cap R'}$ as the number of ``successes'' when we draw $\abs{R'}=d$ times without replacement from a population of size $d^2$ that contains exactly $\abs{R}=d$ successes. That is, $\abs{R\cap R'}\vert R \sim \HG\del{N=d^2, K=d, n=d}$, which is independent of $R$, so that $\abs{R\cap R'} \sim \HG\del{N=d^2, K=d, n=d}$. Our goal is then to bound the MGF of this hypergeometric, which is smaller than the MGF of its binomial counterpart: 
\begin{align*}
    \Eb_{R,R'}\sbr{\exp\del{\lambda\abs{R\cap R'}}} &\leq \Eb\sbr{\exp\del{ \lambda\Bin\del{n=d,p=1/d} }} \\
    &= \del{ 1-\frac{1}{d} + \frac{e^\lambda}{d} }^d \\
    &= \exp\del{d\log\del{ 1 + \frac{e^\lambda-1}{d} }} \\
    &\leq \exp\del{e^\lambda-1}
\end{align*}
Finally, we can combine everything to obtain
\begin{align*}
    \TV\del{P_0^T,\Eb_R\sbr{P_R^T}} &\leq \frac{1}{2}\sqrt{\chi^2\del{\Eb_R\sbr{P_R^T}\middle\Vert P_0^T}} \\
    &\leq \frac{1}{2}\sqrt{ \Eb_{R,R'}\sbr{ \exp\del{\frac{4\epsilon T}{3d}\abs{R\cap R'}} } - 1 } \\
    &\leq \frac{1}{2}\sqrt{ \exp\del{\exp\del{\frac{4\epsilon T}{3d}}-1} - 1 }
\end{align*}
Since the TV distance has to be larger than $1/2$, we can then conclude that
\begin{align*}
    T \geq \frac{3d}{4\epsilon}\log\del{1+\log 2}
\end{align*}
\end{proof}

\begin{remark}[External randomness]
The lower bound of \Cref{thm:SHT-lb-prelim} still applies when the learner has access to an external source of randomness; i.e., it applies to randomized learners.
\end{remark}

We can readily prove our original claim using \Cref{thm:SHT-lb-prelim}. Suppose that $d\leq\frac{1}{2\epsilon}$ and let $\Fc\subset\cbr{0,1}^{\llbr{d^2}}$ be the class of $f=f_0$ along with all functions $f$ that equal $1$ precisely on a subset $R\subset\sbr{d^2}$ of size $d$. Note that $\VC\del{\Fc}=d$.

Let $\Ac$ be an~\eqref{eq:SHT} algorithm with sample complexity $T_\Ac$. We will construct a learner for~\eqref{eq:SHT-Base} with parameters $\del{d,2\epsilon}$. To that end, let our data-generating distribution $P$ be according to one of the hypotheses, and suppose that $\Ac$ wishes to test if $f=f_0$ is $\epsilon$-optimal. Its error is given by $\err\del{f;P} = \Pb\del{Y=1}$, so that we can make the following observations:
\begin{itemize}
    \item Under $H_0$, we have that $f=f^*$ and, thus, $\err\del{f;P}=1/4$.

    \item Under $H_1$ and subset $R$, the functions $f$ and $f^*$ disagree precisely on $R$, so that
    \begin{align*}
        \err\del{f;P} = \del{1-2d\epsilon}\cdot\frac{1}{4} + \del{d^2-d}\cdot\frac{2\epsilon}{d}\cdot\frac{1}{4} + d\cdot\frac{2\epsilon}{d}\cdot\frac{3}{4} = \frac{1}{4}+\epsilon
    \end{align*}
\end{itemize}
Hence, if we run $\Ac$ on $T_\Ac\del{\epsilon,1/4}$ samples, we output $\Ysf$ precisely under $H_0$ with probability $3/4$. In other words, we are able to distinguish $H_0$ and $H_1$. The lower bound of \Cref{thm:SHT-lb-prelim} then implies that $\Ac$ requires a sample size of at least
\begin{align*}
    T_\Ac\del{\epsilon,1/4} \geq \frac{3d}{8\epsilon}\log\del{1+\log 2}
\end{align*}
\subsection{Proof of \Cref{thm:MSHT-ub}}

\ifcolt  \fi

This proof is the natural extension of \Cref{thm:SHT-ub} to multiple distributions. Note that we do not require any union bounds due to the objective. When $d\geq1/\epsilon$, we similarly apply \Cref{lem:sht-test-emp-err} to conclude correctness. The sample size is 
\begin{align*}
    k \cdot T_\Csf\del{\epsilon,\delta} = O\del{ \frac{k\log\del{1/\delta}}{\epsilon^2} }
\end{align*}
When $d<1/\epsilon$, \Cref{lem:sht-learn-to-test} along with a union bound over $i\in\sbr{k}$ again ensures correctness, with a sample size of
\begin{align*}
    k\sbr{ T_\SL\del{ \frac{1}{4}, \frac{\epsilon}{48}, \frac{\delta}{3} } + T_\Lsf\del{\frac{1}{4},\epsilon,\delta} } = O\del{ \frac{kd\log\del{1/\delta}}{\epsilon} }
\end{align*}
\subsection{Proof of \Cref{thm:MSHT-lb}}
In the next result, we construct an~\eqref{eq:MSHT} instance that suffers from the $\Omega\del{kd/\epsilon}$ lower bound when $d\leq\frac{1}{2\epsilon}$, thereby proving our desired hardness claim.

\begin{theorem}
\label{thm:MSHT-lb-prelim}
Let $d\in\Nb$ and $\epsilon\in\del{0,1}$ be such that $d\geq8\epsilon$ and $2d\epsilon\leq1$. Define covariate space $\Xc\coloneqq\llbr{kd^2+k-1}$, composed of $k$ disjoint blocks each of size $d^2+1$, and denote the $i$th block by $\Xc_i \coloneqq \cbr{\del{i-1}\del{d^2+1}+x: x\in\llbr{d^2}}$. Let $\Fc$ be the class of binary-valued functions on $\llbr{kd^2+k-1}$ consisting of
\begin{itemize}
    \item The $f=f_0$ function.

    \item All functions $f$ that equal $1$ precisely on a size-$d$ subset of $\Xc_i\backslash\cbr{\del{i-1}\del{d^2+1}}$ (we exclude the first coordinate) for some $i\in\sbr{k}$.
\end{itemize}
Note that this only requires shattering a set of size at most $d$. Now, define distributions $P_1,\dots,P_k$ over $\Xc\times\cbr{0,1}$ with marginal PMFs
\begin{align*}
    P_i\del{X=x} = \begin{dcases}
        1-2d\epsilon, & x=\del{i-1}\del{d^2+1} \\
        2\epsilon/d, & x\in\cbr{\del{i-1}\del{d^2+1}+1,\dots,i\del{d^2+1}-1}
    \end{dcases}
\end{align*} 
and label noise
\begin{align*}
    P_i\del{f^*\del{X}\neq Y\vert X=x} = \frac{1}{4} \quad\forall x\in\Xc, i\in\sbr{k}
\end{align*}
for some unknown $f^*\in\Fc$.

Any~\eqref{eq:MSHT} algorithm $\Ac$, with confidence parameter $\delta=0.01$, that tests the functions $f_1=\dots=f_k=f_0$ on this instance requires a sample size of 
\begin{align*}
    T_\Ac\del{\epsilon,0.01} \geq \frac{0.015kd}{2\epsilon}
\end{align*}
\end{theorem}

\begin{proof}[Proof of \Cref{thm:MSHT-lb-prelim}]
Consider the following hypotheses:
\begin{equation*}
\begin{aligned}
    &\text{$H_0$: $f^*=f_0$.} \\[6pt]
    &\text{$H_1$: $f^*$ equals $1$ precisely on a size-$d$ subset of $\Xc_i\backslash\cbr{\del{i-1}\del{d^2+1}}$ for some $i\in\sbr{k}$.}
\end{aligned}
\end{equation*}
As in the proof of \Cref{thm:SHT-lb}, we note that
\begin{itemize}
    \item Under $H_0$, we have that $\err\del{f_i;P_i}=1/4$ for every $i\in\sbr{k}$.

    \item Under $H_1$ and distribution $i$, we have that $\err\del{f_i;P_i}=1/4+\epsilon$ and $\err\del{f_j;P_j}=1/4$ for every $j\neq i$.
\end{itemize}
Hence, $\Ac$ must be able to distinguish both hypotheses. Define $T_i$ to be the number of times that $P_i$ is sampled, which is a random quantity since the algorithm can be adaptive. Let $\Pb_0$ be the probability law under $H_0$.

Fix $i\in\sbr{k}$. We will show, by contradiction, that $T_i\gtrsim d/\epsilon$ with constant probability under $H_0$. Assume to the contrary that
\begin{align*}
    \Pb_0\del{T_i\leq \frac{0.05d}{2\epsilon}} \geq 0.7
\end{align*}
We will construct an~\eqref{eq:SHT-Base} algorithm $\Ac_s$ by simulating $\Ac$. Suppose that we are given sample access to a distribution $P$ according to one of the~\eqref{eq:SHT-Base} hypotheses. Let $p_X = \del{1-2d\epsilon,2\epsilon/d,\dots,2\epsilon/d}$ denote the PMF over $\llbr{d^2}$ from \Cref{thm:SHT-lb-prelim} (with $\epsilon$ scaled by $2$). Consider the following strategy:
\begin{enumerate}
    \item Run $\Ac$ on parameters $\del{\epsilon,\delta} = \del{\epsilon,0.1}$. When it samples distribution $j$,
    \begin{itemize}
        \item If $j=i$, sample $\del{X,Y}\sim P$.
        
        \item If $j\neq i$, sample $X\sim p_X$ and $Y\sim\Ber\del{1/4}$.
    \end{itemize}
    Return the data point $\del{X+\del{j-1}\del{d^2+1}, Y}$. The shift on $X$ ensures that it lives in $\Xc_j$. In other words, we set $f^*$ to $0$ outside of $\Xc_i$ and set $P_i$ to be the unknown~\eqref{eq:SHT-Base} instance. Hence, $H_0$ and $H_1$ in both coincide.

    \item We terminate whenever the first of the following occurs:
    \begin{itemize}
        \item If $\Ac$ has sampled distribution $i$ more than $\frac{0.05d}{2\epsilon}$ times, output $\Ysf$ with probability $0.24$.

        \item If $\Ac$ terminates, output $\Dsf_i$.
    \end{itemize}
\end{enumerate}
Note that this process requires at most $\frac{0.05d}{2\epsilon}+1\leq\frac{0.3d}{2\epsilon}$ (since $d\geq8\epsilon$ by assumption) samples from $P$, which beats the lower bound of \Cref{thm:SHT-lb-prelim}. Nevertheless, under the chosen parameters, $\Ac_s$ indeed succeeds. In what follows, we use $\Pb_0$ and $\Pb_1$ to denote the probability law under $H_0$ and $H_1$, respectively. When the probability does not depend on the instance, we drop the subscript. Next, we analyze the output of $\Ac_s$ under each hypothesis.

\paragraph{\underline{$\bm{H_0}$}:}
Suppose that the null hypothesis $H_0$ is true. Then,
\begin{align*}
    \Pb_0\del{\Ac_s=\Ysf} &= \Pb_0\del{\cbr{T_i>\frac{0.05d}{2\epsilon}}\cap\cbr{\Ac_s=\Ysf}} + \Pb_0\del{\cbr{T_i\leq\frac{0.05d}{2\epsilon}}\cap\cbr{\Dsf_i=\Ysf}} \\
    &= \Pb_0\del{T_i>\frac{0.05d}{2\epsilon}}\Pb\del{\Ac_s=\Ysf\middle\vert T_i>\frac{0.05d}{2\epsilon}} \\
    &\hspace{2cm} + \Pb_0\del{\Dsf_i=\Ysf} - \Pb_0\del{\cbr{T_i>\frac{0.05d}{2\epsilon}}\cap\cbr{\Dsf_i=\Ysf}} \\
    &\geq \Pb_0\del{T_i>\frac{0.05d}{2\epsilon}}\Pb\del{\Ac_s=\Ysf\middle\vert T_i>\frac{0.05d}{2\epsilon}} + \Pb_0\del{\Dsf_i=\Ysf} - \Pb_0\del{T_i>\frac{0.05d}{2\epsilon}} \\
    &= \Pb_0\del{\Dsf_i=\Ysf} - \Pb_0\del{T_i>\frac{0.05d}{2\epsilon}} \Pb\del{\Ac_s=\Nsf\middle\vert T_i>\frac{0.05d}{2\epsilon}} \\
    &\geq 0.99 - 0.3\cdot0.76 \\
    &= 0.762
\end{align*}

\paragraph{\underline{$\bm{H_1}$}:}
Suppose that the alternative hypothesis $H_1$ is true. Then,
\begin{align*}
    \Pb_{1}\del{\Ac_s=\Nsf} &= \Pb_{1}\del{\cbr{T_i>\frac{0.05d}{2\epsilon}}\cap\cbr{\Ac_s=\Nsf}} + \Pb_{1}\del{\cbr{T_i\leq\frac{0.05d}{2\epsilon}}\cap\cbr{\Dsf_i=\Nsf}} \\
    &= \Pb_{1}\del{T_i>\frac{0.05d}{2\epsilon}}\Pb\del{\Ac_s=\Nsf\middle\vert T_i>\frac{0.05d}{2\epsilon}} \\
    &\hspace{2cm} + \Pb_{1}\del{\Dsf_i=\Nsf} - \Pb_{1}\del{\cbr{T_i>\frac{0.05d}{2\epsilon}}\cap\cbr{\Dsf_i=\Nsf}} \\
    &\geq \Pb_{1}\del{T_i>\frac{0.05d}{2\epsilon}}\Pb\del{\Ac_s=\Nsf\middle\vert T_i>\frac{0.05d}{2\epsilon}} + \Pb_{1}\del{\Dsf_i=\Nsf} - \Pb_{1}\del{T_i>\frac{0.05d}{2\epsilon}} \\
    &= \Pb_{1}\del{\Dsf_i=\Nsf} - \Pb_{1}\del{T_i>\frac{0.05d}{2\epsilon}} \Pb\del{\Ac_s=\Ysf\middle\vert T_i>\frac{0.05d}{2\epsilon}} \\
    &\geq 0.99 - 0.24 \\
    &= 0.75
\end{align*}

Note that we only required $\Dsf_i$ succeeding with high probability; this allows us to avoid union bounds over $i\in\sbr{k}$. Since $\Ac_s$ always succeeds with a sample size smaller than the lower bound of \Cref{thm:SHT-lb-prelim}, we have thus shown that
\begin{align*}
    \Pb_0\del{T_i> \frac{0.05d}{2\epsilon}} > 0.3
\end{align*}
which we can readily convert into an in-expectation bound:
\begin{align*}
    \Eb_0\sbr{T_i} \geq \Eb_0\sbr{T_i\Ib\cbr{T_i>\frac{0.05d}{2\epsilon}}} \geq \frac{0.05d}{2\epsilon} \Pb_0\del{T_i>\frac{0.05d}{2\epsilon}} > \frac{0.015d}{2\epsilon}
\end{align*}
Since this must hold for every $i\in\sbr{k}$, we can conclude that
\begin{align*}
    T_\Ac\del{\epsilon,0.01} \geq \Eb_0\sbr{\sum_{i=1}^kT_i} \geq \frac{0.015kd}{2\epsilon}
\end{align*}
\end{proof}
\section{Proofs of \Cref{sec:mdl-lb}}

\subsection{Proof of \Cref{lem:MDL-MSHT-ub}}
Correctness of $\Ac$ implies that with probability at least $1-\delta/3$, for all $i\in\sbr{k}$, $\err\del{\hat f_i;P_i}\leq 1/4+\epsilon/48$. \Cref{lem:sht-learn-to-test} then shows that $\Dsf_i$ is correct with probability $1-\delta$, for each $i\in\sbr{k}$.
\subsection{Proof of \Cref{thm:MDL-lb}}
The $\Omega\del{d/\epsilon}$ term follows from the single-distribution learning lower bound (e.g., see Theorem 6.8 of \cite{Shalev-Shwartz2014-zs}). For the second term, we can apply \Cref{lem:MDL-MSHT-ub} to an MDL algorithm $\Ac$ and the~\eqref{eq:MSHT} lower bound of \Cref{thm:MSHT-lb} to conclude that
\begin{align*}
    T_\Ac\del{\epsilon,\frac{0.01}{3}} + \frac{4k\log\del{600}}{\epsilon} = T_{\Ac\to\MSHT}\del{48\epsilon,0.01} \geq \frac{0.015k}{9216} \min\cbr{ \frac{1}{\epsilon^2}, \frac{d}{\epsilon} }
\end{align*}
Rearranging under the lower bound on $\min\cbr{d,1/\epsilon}$ then yields the claim.
\section{Proofs of \Cref{sec:mdl-mass}}
In this section, we prove the~\eqref{eq:MDL-Mass} lower bound of $\Omega\del{k\sqrt{d}/\epsilon}$. We begin by establishing a preliminary lower bound for the testing problem~\eqref{eq:SHT-mass}, and then show how it implies the~\eqref{eq:MDL-Mass} lower bound.

\subsection{Proof of \Cref{thm:SHT-mass-lb}}
We will apply the technique of \textit{Poissonization}: instead of sampling a deterministic number of times $T$, we will sample $N\sim \Pois\del{T}$ pairs $\del{X_t,Y_t}$ and must ensure the same guarantee as~\eqref{eq:SHT-mass}. We will first show a lower bound for the Poissonized variant and subsequently relate it back to the original setup.

Let us define the count of $\del{x,1}$ and $\del{x,0}$ observations:
\begin{align*}
    N_x \coloneqq \sum_{t=1}^N \Ib\cbr{X_t=x, Y_t=1} \quad\text{and}\quad M_x \coloneqq \sum_{t=1}^N \Ib\cbr{X_t=x, Y_t=0}
\end{align*}
We will work with random $\qbf$, independent of $N$, in which case we can ensure that
\begin{align*}
    N_x\vert \qbf \sim \Pois\del{Tp_X\del{x}q_x} \quad\text{and}\quad M_x\vert \qbf \sim \Pois\del{Tp_X\del{x}\del{1-q_x}}
\end{align*}
and are independent conditional on $\qbf$. Also, note that $N_x+M_x = \sum_{t=1}^N \Ib\cbr{X_t=x} \sim \Pois\del{Tp_X\del{x}}$ and is independent of $\qbf$.

The next step is to construct appropriate biases for each hypothesis. For $H_0$, we will simply define the constant vector $\qbf_0 \coloneqq \del{0.465,\dots,0.465}$. For $H_1$, we define the prior $\mu_1 \coloneqq 0.75 \cdot \Ber\del{0.62}$ and set $\qbf_1 \sim \delta_{0.465}\times \mu_1^d$. In other words, the $0$th coordinate is always $0.465$, and the rest are sampled i.i.d. from $\mu_1$. One important consequence of this construction is that the first moments match:
\begin{align*}
    \Eb_{q\sim\mu_1}\sbr{q} = 0.75\cdot0.62 = 0.465 = \Eb_{q\sim\mu_0}\sbr{q}
\end{align*}
where $\mu_0 \coloneqq \delta_{0.465}$.

While $\qbf_0$ falls under $H_0$ with probability $1$, we can only ensure a high-probability guarantee for $\qbf_1$. Since
\begin{align*}
    \Eb\sbr{\Delta_{\qbf_1}} = \frac{\epsilon}{d}\cdot d\cdot0.62\cdot0.5 = 0.31\epsilon
\end{align*}
we can apply Hoeffding's inequality (note that $\Delta_{\qbf_1}\in\sbr{0,\epsilon/2}$) to conclude that
\begin{align*}
    \Pb\del{ \Delta_{\qbf_1} \geq 0.3\epsilon } &= \Pb\del{ \Delta_{\qbf_1}-\Eb\sbr{\Delta_{\qbf_1}} \geq -0.01\epsilon } \geq 1-\exp\del{-0.0008d} \geq \frac{3}{4}
\end{align*}
Under hypothesis $h\in\cbr{0,1}$, our decision is based on observations $O_h \coloneqq \del{N,\del{X_t,Y_t}_{t=1}^N}\sim P_h^O$. Our objective requires that $P_0^O\del{\Ysf} \geq 3/4$, since the bias vector is deterministic under our constructed $H_0$. Furthermore, defining the event $A\coloneqq\cbr{\Delta_{\qbf_1}\geq0.3\epsilon}$, we must also satisfy
\begin{align*}
    P_1^O\del{\Ysf} = \underbrace{ P_1^O\del{\Ysf\vert A} }_{\leq1/4} P_1^O\del{A} + P_1^O\del{\Ysf\middle\vert A^c} \underbrace{ P_1^O\del{A^c} }_{\leq1/4} \leq \frac{1}{2}
\end{align*}
As a consequence, we get the lower bound
\begin{align*}
    \TV\del{P_0^O, P_1^O} \geq P_0^O\del{\Ysf}-P_1^O\del{\Ysf} \geq \frac{1}{4}
\end{align*}
To further bound the TV distance, we instead switch our attention to the counts $C_h \coloneqq \del{M_x,N_x}_{x=0}^d \sim P_h^C$. To see why they suffice, note that $C_h$ is a deterministic function of $O_h$. Moreover, we simply pass $C_h$ through a Markov kernel that is \textit{independent} of $\qbf_h$ to obtain $O_h$ (choose an ordering of the $\del{X_t,Y_t}$ uniformly at random). This implies that
\begin{align*}
    \TV\del{P_0^O, P_1^O} = \TV\del{P_0^C, P_1^C}
\end{align*}
Next, we upper bound the right-hand side.

\begin{lemma}
\label{lem:TV-count-ub}
We have that
\begin{align*}
    \TV\del{P_0^C, P_1^C} \leq \frac{T^2\epsilon^2}{2d}
\end{align*}
\end{lemma}
\begin{proof}[Proof of \Cref{lem:TV-count-ub}]
We begin with the observation that the coordinates of $C_h$ are independent, and the $0$th coordinate of both $C_0$ and $C_1$ is the same by construction. Let $\tilde C_h\sim P_h^{\tilde C}$ denote coordinates $\sbr{d}$ of $C_h$, so that $\TV\del{P_0^C, P_1^C} = \TV\del{P_0^{\tilde C},P_1^{\tilde C}}$. To bound this, we note that $\tilde C_h \iid P_{\mu_h}$, where
\begin{align*}
    P_q \coloneqq \Pois\del{\frac{T\epsilon q}{d}} \times \Pois\del{\frac{T\epsilon\del{1-q}}{d}} \quad\text{and}\quad P_{q\sim\mu} \coloneqq \Eb_\mu\sbr{P_q}
\end{align*}
Using subadditivity of TV, we then have that $\TV\del{P_0^{\tilde C},P_1^{\tilde C}} \leq d \TV\del{P_{\mu_0},P_{\mu_1}}$. To upper bound the right-hand side, let $\lambda=T\epsilon/d$ and note that
\begin{align*}
    P_q\del{n,m} = \del{\frac{\del{\lambda q}^ne^{-\lambda q}}{n!}} \del{\frac{\del{\lambda\del{1-q}}^me^{-\lambda \del{1-q}}}{m!}} = \frac{\lambda^{n+m} q^n\del{1-q}^m e^{-\lambda}}{n!m!}
\end{align*}
In particular,
\begin{align*}
    P_q\del{0,0} = e^{-\lambda}, \quad P_q\del{1,0} = \lambda e^{-\lambda}q, \quad P_q\del{0,1} = \lambda e^{-\lambda}\del{1-q}
\end{align*}
Since $\Eb_{\mu_0}\sbr{q}=\Eb_{\mu_1}\sbr{q}$, we then know that
\begin{align*}
    P_{\mu_0}\del{n,m} = P_{\mu_1}\del{n,m} \quad\forall \del{n,m}\in\cbr{\del{0,0},\del{1,0},\del{0,1}}
\end{align*}
As a result, we get that
\begin{align*}
    \TV\del{ P_{\mu_0}, P_{\mu_1} } &= \frac{1}{2} \sum_{n+m\geq 2} \abs{ P_{\mu_0}\del{n,m} - P_{\mu_1}\del{n,m} } \\
    &\leq \frac{1}{2} \sbr{ P_{\mu_0}\del{N+M\geq 2} + P_{\mu_1}\del{N+M\geq 2} } \\
    &= \Pb\del{\Pois\del{\lambda}\geq 2}
\end{align*}
where in the last line we used the fact that $N+M\sim\Pois\del{\lambda}$ is independent of $q$. Lastly, we must bound the Poisson tail above. To start, note that
\begin{align*}
    \Pb\del{\Pois\del{\lambda}\geq 2} = 1 - \Pb\del{\Pois\del{\lambda}=0} - \Pb\del{\Pois\del{\lambda}=1} = 1 - e^{-\lambda}\del{1+\lambda}
\end{align*}
Define $g\del{\lambda}\coloneqq e^{-\lambda}\del{1+\lambda}$ and note that $g'\del{\lambda}=-\lambda e^{-\lambda}$. Then,
\begin{align*}
    1 - e^{-\lambda}\del{1+\lambda} &= g\del{0} - g\del{\lambda} = \int_{0}^\lambda -g'\del{t} dt = \int_0^\lambda te^{-t} dt \leq \int_0^\lambda t dt = \frac{\lambda^2}{2}
\end{align*}
Putting everything together, we conclude that
\begin{align*}
    \TV\del{P_0^C, P_1^C} &\leq \frac{d\lambda^2}{2} = \frac{T^2\epsilon^2}{2d}
\end{align*}
\end{proof}

From \Cref{lem:TV-count-ub}, we can establish a lower bound on the Poissonized setting:
\begin{align*}
    \frac{T^2\epsilon^2}{2d} \geq \frac{1}{4} \Longrightarrow T \geq \frac{\sqrt{d}}{\sqrt{2}\epsilon}
\end{align*}
To obtain a similar bound for our original problem, we rely on the following Poisson tail bound~\cite{Mitzenmacher2005Probability}: for $\lambda\geq 12\log\del{2/\delta}$,
\begin{align*}
    \Pb\del{ \Pois\del{\lambda} \in\sbr{\frac{\lambda}{2}, \frac{3\lambda}{2}}} \geq 1-\delta
\end{align*}
Suppose that a tester $\Ac$ solves~\eqref{eq:SHT-mass} with at most $T$ samples. Consider the following strategy $\Ac_P$ for the Poissonized variant:
\begin{itemize}
    \item Sample $N\sim\Pois\del{2T'}$, where $T' = \max\cbr{T,14}$.
    \item If $N\geq T$, run $\Ac$ on the first $T$ samples and output its answer. Otherwise, decide based on a fair coin flip.
\end{itemize}
Since $2T' \geq 28 \geq 12\log\del{2/0.2}$, the Poisson tail bound implies that
\begin{align*}
    \Pb\del{N\geq T} \geq \Pb\del{ N \in \sbr{T',3T'} } \geq 0.8
\end{align*}
Let $\Pb_h$ denote the probability measure under any instance of hypothesis $h\in\cbr{0,1}$, and define the event $B\coloneqq\cbr{N\geq T}$. Then,
\begin{align*}
    \Pb_0\del{\Ac_P=\Ysf} &= \Pb_0\del{\Ac=\Ysf\vert B}\Pb\del{B} + \Pb_0\del{\Ac_P=\Ysf\middle\vert B^c}\Pb\del{B^c} \\
    &= \Pb_0\del{\Ac_P=\Ysf\middle\vert B^c} + \Pb\del{B}\del{\Pb_0\del{\Ac=\Ysf\vert B} - \Pb_0\del{\Ac_P=\Ysf\middle\vert B^c}} \\
    &\geq 0.5 + 0.8\del{0.9-0.5} \\
    &\geq 0.75
\end{align*}
By a symmetric argument, $\Pb_0\del{\Ac_P=\Nsf}\geq0.75$. In other words, we solve the Poissonized problem and its lower bound implies that
\begin{align*}
    2T' \geq \frac{\sqrt{d}}{\sqrt{2}\epsilon} \Longrightarrow T \geq \frac{\sqrt{d}}{2\sqrt{2}\epsilon}
\end{align*}
where we used the assumption $d\geq1750$, so that $\frac{\sqrt{d}}{\sqrt{2}\epsilon}>28$ and, thus, $T'=T$.

\subsection{Proof of \Cref{thm:MDL-Mass-lb}}
To start, we note that our construction requires that $f^*$ always equals $0$ on the first coordinate of each $\Xc_i$. Then, if we apply \Cref{lem:noisy-class-exc-err} on $f_0$, we can see that
\begin{align*}
    \Delta_i \coloneqq \err\del{f_0;P_i}-\eta_i^* = \sum_{x\in\Xc:f^*\del{x}=1} \del{1-2\eta_i^*\del{x}} P_i\del{X=x} = \frac{\epsilon}{d} \sum_{x\in\Xc_i:q_x^i\geq 1/2} \del{2q_x^i-1}
\end{align*}
where $q_x^i \coloneqq P_i\del{Y=1\vert X=x}$.

Consider the setting where $f^*=f_0$ and $\eta_i^*\del{x} = q_x^i = 0.465$ for all $x\in\Xc$ and $i\in\sbr{k}$. Let $\Pb_0$ denote the probability law under this environment. We will show that any successful~\eqref{eq:MDL-Mass} algorithm must sample $\Omega\del{\sqrt{d}/\epsilon}$ times from each distribution under $\Pb_0$ with high probability.

To do this, we will solve an \eqref{eq:SHT-mass} hard instance $P$ by simulating a~\eqref{eq:MDL-Mass} algorithm $\Ac$. Here, $P$ is a distribution over $\del{X,Y}$ where $p_X=\del{1-\epsilon, \epsilon/d,\dots,\epsilon/d}$ on $\llbr{d}$ and $Y\vert X=x\sim\Ber\del{q_x}$. Recall our testing objective:
\begin{itemize}
    \item $H_0$: $q_x=0.465$ for all $x\in\llbr{d}$.
    \item $H_1$: $\Delta_\qbf = \frac{\epsilon}{d} \sum_{x\in\llbr{d}:q_x\geq1/2} \del{2q_x-1} \geq0.3\epsilon$.
\end{itemize}
Fix some $i\in\sbr{k}$ and let $T_i$ be number of times that $\Ac$ samples $P_i$. We will show, by contradiction, that $T_i\gtrsim \sqrt{d}/\epsilon$ with high probability under $\Pb_0$. Let $B_i\coloneqq \cbr{T_i\leq 0.05\sqrt{d}/\epsilon}$ and assume to the contrary that 
\begin{align*}
    \Pb_0\del{B_i} \geq 0.85
\end{align*}
We will construct an~\eqref{eq:SHT-mass} $\Ac_s$ strategy as follows:
\begin{enumerate}
    \item Run $\Ac$ on parameters $\del{\eta,\epsilon,\delta} = \del{0.49,\epsilon/192,0.1/3}$. When it samples distribution $j\in\sbr{k}$,
    \begin{itemize}
        \item If $j=i$, sample $\del{X,Y}\sim P$.

        \item If $j\neq i$, sample $X\sim p_X$ and $Y\sim\Ber\del{0.465}$.
    \end{itemize}
    To ensure that $X\in\Xc_j$, we shift $X$ by $\del{j-1}\del{d+1}$ before returning it to $\Ac$. In other words, we set $P_i$ to $P$, with an appropriate shift, and set $f^*$ to $0$ outside of $\Xc_i$. By construction, this aligns with our MDL setup.

    \item Terminate when the first of the following occurs:
    \begin{itemize}
        \item If $T_i$ exceeds $0.05\sqrt{d}/\epsilon$, output $\Nsf$.

        \item If $\Ac$ terminates and outputs $\hat f_i$, sample $S\iid P$ of size 
        \begin{align*}
            \abs{S} = T_\Lsf\del{0.49,\epsilon/4,0.1} = \frac{384}{0.02\epsilon}\log\del{60} \leq \frac{80000}{\epsilon}
        \end{align*}
        shift the $X$'s appropriately, and output $\Ysf$ if and only if event
        \begin{align*}
            E_i \coloneqq \cbr{ \abs{ \emperr\del{f_0;S}-\emperr\del{\hat f_i;S} } \leq \frac{\epsilon}{12} }
        \end{align*}
        occurs.
    \end{itemize}
\end{enumerate}
Note that the suboptimality gaps coincide: $\Delta_i = \Delta_\qbf$. In addition, correctness of $\Ac$ implies, via \Cref{lem:sht-learn-to-test}, that with probability at least $0.9$,
\begin{itemize}
    \item Under $H_0$, we have that $\Delta_i=0$, so that $E_i$ occurs.

    \item Under $H_1$, we have that $\Delta_i\geq0.3\epsilon$, so that $E_i^c$ occurs.
\end{itemize}
Importantly, this statement is \textit{unconditional}; that is, under no assumption of $B_i$ occurring. Let $\Pb_i$ denote the probability law under $H_i$, where we note that $\Pb_0$ coincides with our earlier definition.

\paragraph{\underline{$\bm{H_0}$}:}
\begin{align*}
    \Pb_0\del{\Ac_s=\Ysf} &= \Pb_0\del{B_i \cap E_i} = \Pb_0\del{E_i} - \Pb_0\del{B_i^c \cap E_i} \geq \Pb_0\del{E_i} - \Pb_0\del{B_i^c} \geq 0.75
\end{align*}

\paragraph{\underline{$\bm{H_1}$}:}
\begin{align*}
    \Pb_1\del{\Ac_s=\Nsf} &= \Pb_1\del{B_i^c} + \Pb_1\del{B_i \cap E_i^c} = \Pb_1\del{B_i^c} + \Pb_1\del{E_i^c} - \Pb_1\del{B_i^c \cap E_i^c} \geq \Pb_1\del{E_i^c} \geq 0.9
\end{align*}
That is, we solve~\eqref{eq:SHT-mass} with a sample size of at most
\begin{align*}
    \frac{0.05\sqrt{d}}{\epsilon} + 1 + \frac{80000}{\epsilon} \leq \frac{\sqrt{d}}{2\sqrt{2}\epsilon}
\end{align*}
beating the lower bound of \Cref{thm:SHT-mass-lb}. This contradiction ensures that
\begin{align*}
    \Pb_0\del{T_i>\frac{0.05\sqrt{d}}{\epsilon}} > 0.15 \Longrightarrow \Eb_0\sbr{T_i} \geq \Eb_0\sbr{T_i \Ib\cbr{{T_i>\frac{0.05\sqrt{d}}{\epsilon}}}} \geq \frac{0.0075\sqrt{d}}{\epsilon}
\end{align*}
Since this holds for each $i\in\sbr{k}$, we finally get that
\begin{align*}
    T_\Ac\del{\frac{\epsilon}{192},\frac{0.1}{3}} \geq \Eb_0\sbr{\sum_{i=1}^k T_i} \geq \frac{0.0075k\sqrt{d}}{\epsilon}
\end{align*}

\end{document}